\documentclass{article}

\usepackage{microtype}
\usepackage{graphicx}
\usepackage{subfigure}
\usepackage{booktabs} 

\usepackage{hyperref}

\usepackage[accepted]{icml2023}

\usepackage{amsmath}
\usepackage{amssymb}
\usepackage{mathtools}
\usepackage{amsthm}

\usepackage[capitalize,noabbrev]{cleveref}

\theoremstyle{plain}
\newtheorem{theorem}{Theorem}[section]
\newtheorem{proposition}[theorem]{Proposition}
\newtheorem{lemma}[theorem]{Lemma}
\newtheorem{corollary}[theorem]{Corollary}
\theoremstyle{definition}

\newtheorem{assumption}[theorem]{Assumption}
\theoremstyle{remark}
\newtheorem{remark}[theorem]{Remark}

\usepackage[textsize=tiny]{todonotes}

\usepackage{mathrsfs}
\usepackage{enumerate}

\icmltitlerunning{On the Optimality of Misspecified Kernel Ridge Regression}

\begin{document}

\twocolumn[
\icmltitle{On the Optimality of Misspecified Kernel Ridge Regression}




\begin{icmlauthorlist}
\icmlauthor{Haobo Zhang}{yyy}
\icmlauthor{Yicheng Li}{yyy}
\icmlauthor{Weihao Lu}{yyy}
\icmlauthor{Qian Lin}{yyy,comp}
\end{icmlauthorlist}

\icmlaffiliation{yyy}{Center for Statistical Science, Department of Industrial Engineering, Tsinghua University}
\icmlaffiliation{comp}{Beijing Academy of Artificial Intelligence, Beijing, China}

\icmlcorrespondingauthor{Qian Lin}{qianlin@tsinghua.edu.cn}
\icmlkeywords{Kernel ridge regression, Misspecified, Reproducing kernel Hilbert space, Sobolev space, minimax optimality}

\vskip 0.3in
]



\printAffiliationsAndNotice{}  

\begin{abstract}
In the misspecified kernel ridge regression problem, researchers usually assume the underground true function $f_{\rho}^{*} \in [\mathcal{H}]^{s}$,  a less-smooth interpolation space of a reproducing kernel Hilbert space (RKHS) $\mathcal{H}$ for some $s\in (0,1)$. 
The existing minimax optimal results require   $\|f_{\rho}^{*}\|_{L^{\infty}}<\infty$ which implicitly requires  $s > \alpha_{0}$ where $\alpha_{0}\in (0,1)$ is the embedding index, a constant depending on $\mathcal{H}$. 
Whether the KRR is optimal for all $s\in (0,1)$ is an outstanding problem lasting for years. 
In this paper, we show that KRR is minimax optimal for any $s\in (0,1)$ when the $\mathcal{H}$ is a Sobolev RKHS. 
\end{abstract}

\section{Introduction}
Suppose that the samples $\{ (x_{i}, y_{i}) \}_{i=1}^{n}$ are i.i.d. sampled from an unknown distribution $\rho$ on $\mathcal{X} \times \mathcal{Y}$, where $\mathcal{X} \subseteq \mathbb{R}^{d}$ and $\mathcal{Y} \subseteq \mathbb{R}$. The regression problem aims
to find a function $\hat{f}$ such that the risk
\begin{displaymath}
  \mathcal{E}(\hat{f}) = \mathbb{E}_{(x,y) \sim \rho} \left[ \left( \hat{f}(x) - y \right)^{2} \right]
\end{displaymath}
is relatively small. It is well known that the conditional mean function given by $f_{\rho}^*(x) \coloneqq \mathbb{E}_{\rho}[~y \;|\; x~] = \int_{\mathcal{Y}} y \mathrm{d} \rho(y|x)$ minimizes the risk $\mathcal{E}(f) $. 
Therefore, 
we may focus on establishing the convergence rate (either in expectation or in probability) for the excess risk (generalization error)
\begin{equation}\label{def of gen}
    \mathbb{E}_{x \sim \mu} \left[ \left( \hat{f}(x) - f_{\rho}^{*}(x) \right)^{2} \right],
\end{equation}
where $ \mu $ is the marginal distribution of $\rho$ on $\mathcal{X}$. 

In the non-parametric regression settings, researchers often assume that $f_{\rho}^*(x)$ falls into a class of functions with certain structures and develop non-parametric methods to obtain the estimator $\hat{f}$. One of the most popular non-parametric regression methods, the kernel method, aims to estimate $f_{\rho}^*$ using candidate functions from a reproducing kernel Hilbert space (RKHS) $\mathcal{H}$, a separable Hilbert space associated to a kernel function $k$ defined on $\mathcal{X}$, e.g., \citet{Kohler2001NonparametricRE, Cucker2001OnTM, Caponnetto2007OptimalRF, Steinwart2008SupportVM}. This paper focuses on the kernel ridge regression (KRR), which constructs an estimator $\hat{f}_{\lambda}$ by solving the  penalized least square problem
\begin{equation}\label{krr estimator}
    \hat{f}_\lambda = \underset{f \in \mathcal{H}}{\arg \min } \left(\frac{1}{n} \sum_{i=1}^n\left(y_i-f\left(x_i\right)\right)^2+\lambda\|f\|_{\mathcal{H}}^2\right),
\end{equation}
where $\lambda > 0$ is referred to as the regularization parameter. 

Since the minimax optimality of KRR has been proved for $f_{\rho}^{*} \in [\mathcal{H}]^{s}, 1 \le s \le 2$ \citep{Caponnetto2007OptimalRF}, a large body of literature has studied the convergence rate of the generalization error of misspecified KRR ($ f_{\rho}^{*} \notin \mathcal{H}$) and whether the rate is optimal in the minimax sense. 
It turns out that the eigenvalue decay rate ($\beta > 1$), the source condition ($s>0$), and the embedding index ($\alpha_{0} < 1$) of the RKHS jointly determine the convergence behavior of the misspecified KRR (see Section \ref{section assumption} for definitions). 
If we only assume that $f_{\rho}^{*}$ belongs to an interpolation space $[\mathcal{H}]^{s}$ of the RKHS $\mathcal{H}$ for some $s > 0$, the well-known information-theoretic lower bound shows that the minimax lower bound is $n^{-\frac{s \beta}{s \beta + 1}} $. The state-of-the-art work \citet{fischer2020_SobolevNorm} has already shown that when $ \alpha_{0} < s \le 2$, the upper bound of the convergence rate of KRR is $ n^{-\frac{s \beta}{s \beta + 1}}$ and hence is optimal. 
However, when $ f_{\rho}^{*} \in [\mathcal{H}]^{s}$ for some $0 < s \le \alpha_{0}$, all the existing works need an additional boundedness assumption on $f_{\rho}^{*}$ to prove the same upper bound $ n^{-\frac{s \beta}{s \beta + 1}}$. 
The boundedness assumption will result in a smaller function space, i.e., $ [\mathcal{H}]^{s} \cap L^{\infty}(\mathcal{X,\mu}) \subsetneqq [\mathcal{H}]^{s}$ when $ s \le \alpha_{0}$. \citet{fischer2020_SobolevNorm} further reveals that the minimax rate of the excess risk associated to the smaller function space is larger than  $ n^{-\frac{\alpha \beta}{ \alpha \beta + 1}}$ for any $\alpha > \alpha_{0}$.
This lower bound of the minimax rate is smaller than the upper bound of the convergence rate and hence they can not prove the minimax optimality of KRR when $s \le \alpha_{0}$.


It has been an outstanding problem for years whether KRR is minimax optimal for all the $s\in (0,1)$ \cite{PillaudVivien2018StatisticalOO, fischer2020_SobolevNorm, Liu2022StatisticalOO}. 
This paper concludes that, for Sobolev RKHS, KRR is optimal for all $ 0 < s < 1$. Thus, we know that KRR is optimal for Sobolev RKHS and all $0<s\leq 2$. Together with a recent work on the saturation effect where KRR can not be optimal for $s > 2$ \cite{li2023saturation}, the optimality of KRR for Sobolev RKHS is well understood. 

\subsection{Related work}
Kernel ridge regression has been studied as a special kind of spectral regularization algorithm \cite{rosasco2005_SpectralMethods, Caponnetto2006OptimalRF,  bauer2007_RegularizationAlgorithms, gerfo2008_SpectralAlgorithms, mendelson2010_RegularizationKernel}. In large part of the literature, the `hardness' of the regression problem is determined by two parameters: 1. the source condition ($s$), which characterizes a function's relative `smoothness' with respect to the RKHS; 2. the eigenvalue decay rate ($\beta$) (or capacity, effective dimension equivalently), which characterizes the RKHS itself. These two parameters divide the convergence behavior of KRR or spectral regularization algorithm into different regimes and lead to different convergence rates \citep[etc.]{dicker2017_KernelRidge, blanchard2018_OptimalRates, lin2018_OptimalRates, lin2020_OptimalConvergence,li2023saturation}. \citet{Caponnetto2007OptimalRF} first proves the optimality when $1 \le s \le 2$ and \citet{lin2018_OptimalRates} extends the desired upper bound of the convergence rate to the regime $ s + \frac{1}{\beta} > 1$.

KRR in the misspecified case ($f_{\rho}^{*} \notin \mathcal{H} ~\text{or}~ s < 1$) has also been discussed by another important line of work which considers the embedding index ($\alpha_{0}$) of the RKHS and performs refined analysis \cite{steinwart2009_OptimalRates, dicker2017_KernelRidge, PillaudVivien2018StatisticalOO,fischer2020_SobolevNorm,  Celisse2020AnalyzingTD,Li2022OptimalRF}. The desired upper bound $n^{-\frac{s \beta}{s \beta + 1}}$ is extended to the regime $s + \frac{1}{\beta} > \alpha_{0}$, and the minimax optimality is extended to the regime $s > \alpha_{0}$. 
It is worth pointing out that when $f_{\rho}^{*}$ falls into a less-smooth interpolation space which does not imply the boundedness of functions therein, all existing works (either considering embedding index or not) require an additional boundedness assumption, i.e., $ \| f_{\rho}^{*} \|_{L^{\infty}(\mathcal{X},\mu)} \le B_{\infty} < \infty$ \citep[etc]{lin2020_OptimalConvergence, fischer2020_SobolevNorm, Talwai2022OptimalLR,Li2022OptimalRF}. As discussed in the introduction, this will lead to the suboptimality in the $s \le \alpha_{0}$ regime. 

This paper follows the line of work that considers the embedding index, refines the proof by handling the additional boundedness assumption, and solves the optimality problem for Sobolev RKHS and all $ s \in (0,1)$. In addition, our technical improvement also sheds light on the optimality for general RKHS. Specifically, we replace the boundedness assumption with a far weaker $L^{q}$-integrability assumption, which turns out to be reasonable for many RKHS. Note that our results focus on the most frequently used $L^{2}$ convergence rate for KRR, and it can be easily extended to $[\mathcal{H}]^{\gamma}, \gamma \ge 0$  convergence rate (e.g., \citealt{fischer2020_SobolevNorm}) and general spectral regularization algorithms (e.g., \citealt{lin2018_OptimalRates}). 

\section{Preliminaries}
\subsection{Basic concepts}
Let $\mathcal{X} \subseteq \mathbb{R}^{d}$ be the input space and $ \mathcal{Y} \subseteq \mathbb{R}$ be the output space. Let $ \rho $ be an unknown probability distribution on $\mathcal{X} \times \mathcal{Y}$ satisfying $ \int_{\mathcal{X} \times \mathcal{Y}} y^{2} \mathrm{d}\rho(x,y) <\infty$, and denote the corresponding marginal distribution on $ \mathcal{X} $ as $\mu$. We use $L^{p}(\mathcal{X},\mu)$ (in short $L^{p}$) to represent the $L^{p}$-spaces. Then the generalization error \eqref{def of gen} can be written as 
\begin{displaymath}
  \left\| \hat{f} - f_{\rho}^{*} \right\|_{L^{2}}^{2}.
\end{displaymath}


Throughout the paper, we denote $\mathcal{H}$ as a separable RKHS on $\mathcal{X}$ with respect to a continuous and bounded kernel function $k$ satisfying 
\begin{displaymath}
  \sup\limits_{x \in \mathcal{X}} k(x,x) \le \kappa^{2}.
\end{displaymath}
Define the integral operator $ T: L^{2}(\mathcal{X},\mu) \to L^{2}(\mathcal{X},\mu)$ as
\begin{equation}\label{def of T}
    (T f)(x):=\int_{\mathcal{X}} k(x, y) f(y) \mathrm{d} \mu(y).
\end{equation}
It is well known that $T$ is a positive, self-adjoint, trace-class, and hence a compact operator \cite{steinwart2012_MercerTheorem}. The spectral theorem for self-adjoint compact operators and Mercer's decomposition theorem yield that 
\begin{align}
    k(x, y) &= \sum_{i \in N} \lambda_i e_i(x) e_i(y), \notag \\
    T &= \sum_{i \in N} \lambda_i\left\langle\cdot, e_i\right\rangle_{L^2} e_i, \notag
\end{align}
where $N$ is an at most countable set, the eigenvalues $\{ \lambda_{i} \}_{i \in N} \subseteq (0,\infty)$ is a non-increasing summable sequence, and $ \{ e_{i} \}_{i \in N}$ are the corresponding eigenfunctions. Denote the samples as $X=\left( x_{1},\cdots,x_{n} \right)$ and $\textbf{y}=\left( y_{1},\cdots,y_{n} \right)^{\prime}$. The representer theorem (see, e.g., \citealt{Steinwart2008SupportVM}) gives an explicit expression of the KRR estimator defined by \eqref{krr estimator}, i.e., 
\begin{displaymath}
   \hat{f}_{\lambda}(x) = \mathbb{K}(x, X)(\mathbb{K}(X, X)+n \lambda I)^{-1} \mathbf{y},
\end{displaymath}
where 
\begin{displaymath}
  \mathbb{K}(X, X)=\left(k\left(x_i, x_j\right)\right)_{n \times n},
\end{displaymath}
and
\begin{displaymath}
   \mathbb{K}(x, X)=\left(k\left(x, x_1\right), \cdots, k\left(x, x_n\right)\right).
\end{displaymath}

We also need to introduce the interpolation spaces of RKHS. For any $ s \ge 0$, the fractional power integral operator $T^{s}: L^{2}(\mathcal{X},\mu) \to L^{2}(\mathcal{X},\mu)$ is defined as 
\begin{displaymath}
  T^{s}(f)=\sum_{i \in N} \lambda_i^{s} \left\langle f, e_i\right\rangle_{L^2} e_i,
\end{displaymath}
and the interpolation space $[\mathcal{H}]^{s}, s \ge 0$ of $\mathcal{H}$ is defined as 
\begin{equation}\label{def interpolation space}
    [\mathcal{H}]^s \coloneqq \operatorname{Ran} T^{\frac{s}{2}}=\left\{\sum_{i \in N} \lambda_i^{\frac{s}{2}}  a_i e_i \mid \sum\limits_{i \in N} a_{i}^{2} < \infty \right\},
\end{equation}
with the inner product 
\begin{equation}\label{def of interpolation norm}
    \langle f, g\rangle_{[\mathcal{H}]^s}=\left\langle T^{-\frac{s}{2}} f, T^{-\frac{s}{2}} g\right\rangle_{L^2} .
\end{equation}
It is easy to show that $[\mathcal{H}]^s $ is also a separable Hilbert space with orthogonal basis $ \{ \lambda_{i}^{\frac{s}{2}} e_{i}\}_{i \in N}$. Specially, we have $[\mathcal{H}]^0 \subseteq L^{2}(\mathcal{X},\mu) $ and $[\mathcal{H}]^1 \subseteq \mathcal{H}$. For $0 < s_{1} < s_{2}$, the embeddings $ [\mathcal{H}]^{s_{2}} \hookrightarrow[\mathcal{H}]^{s_{1}} \hookrightarrow[\mathcal{H}]^0 $ exist and are compact \cite{fischer2020_SobolevNorm}. For the functions in $[\mathcal{H}]^{s}$ with larger $s$, we say they have higher regularity (smoothness) with respect to the RKHS. 

As an example, the Sobolev space $H^{m}(\mathcal{X})$ is an RKHS if $m > \frac{d}{2}$, and its interpolation space is still a Sobolev space given by $ [H^{m}(\mathcal{X})]^s \cong H^{m s}(\mathcal{X}), \forall s>0 $, see Section \ref{section sobolev} for detailed discussions.

\subsection{Assumptions}\label{section assumption}
This subsection lists the standard assumptions that frequently appeared in related literature.
\vspace{6pt}
\begin{assumption}[Eigenvalue decay rate (EDR)]\label{ass EDR}
 Suppose that the eigenvalue decay rate (EDR) of $\mathcal{H}$ is $\beta > 1$, i.e, there are positive constants $c$ and $C$ such that 
 \begin{displaymath}
   c i^{- \beta} \le \lambda_{i} \le C i^{-\beta}, \quad  \forall i \in N.
 \end{displaymath}
\end{assumption}
Note that the eigenvalues $\lambda_{i}$ and EDR are only determined by the marginal distribution $\mu$ and the RKHS $\mathcal{H}$. The polynomial eigenvalue decay rate assumption is standard in related literature and is also referred to as the capacity condition or effective dimension
condition.
\vspace{6pt}
\begin{assumption}[Embedding index]\label{assumption embedding}
 We say that $[\mathcal{H}]^\alpha$ has the embedding property  for some $\alpha\in [\frac{1}{\beta},1]$, if there is a constant $0 < A < \infty$ such that
 \begin{equation}\label{embedding property 1.12}
     \left\|[\mathcal{H}]^\alpha \hookrightarrow L^{\infty}(\mathcal{X},\mu)\right\| \leq A,
 \end{equation}
where $\|\cdot\| $ denotes the operator norm of the embedding.
Then we define the \textit{embedding index} of an RKHS $\mathcal{H}$ as
\begin{displaymath}
  \alpha_{0} = \inf\left\{ \alpha :  [\mathcal{H}]^\alpha~ \text{has the embedding property}  \right\}.
\end{displaymath}
\end{assumption}

In fact, for any $\alpha > 0$, we can define $M_{\alpha} $ as the smallest constant $A > 0$ such that 
\begin{displaymath}
  \sum_{i \in N} \lambda_i^\alpha e_i^2(x) \leq A^2, \quad \mu \text {-a.e. } x \in \mathcal{X},
\end{displaymath}
if there is no such constant, set $ M_{\alpha} = \infty$. Then \citet[Theorem 9]{fischer2020_SobolevNorm} shows that for $ \alpha > 0$,
\begin{displaymath}
  \left\|[\mathcal{H}]^\alpha \hookrightarrow L^{\infty}(\mathcal{X},\mu)\right\|=M_{\alpha}.
\end{displaymath}
    
The larger $\alpha$ is, the weaker the embedding property is. Note that since $ \sup_{x \in \mathcal{X}} k(x,x) \le \kappa^{2} $, $M_{\alpha} \le \kappa < \infty$ is always true for $\alpha \ge 1$. In addition, \citet[Lemma 10]{fischer2020_SobolevNorm} also shows that $\alpha$ can not be less than $\frac{1}{\beta}$. 
\vspace{6pt}

Note that the embedding property \eqref{embedding property 1.12} holds for any $\alpha > \alpha_{0}$. This directly implies that all the functions in $[\mathcal{H}]^\alpha$ are $\mu \text {-a.e.}$ bounded, $\alpha > \alpha_{0}$. However, the embedding property may not hold for $\alpha = \alpha_{0}$.

\vspace{6pt}
\begin{assumption}[Source condition]\label{ass source condition}
  For $s > 0 $, there is a constant $R > 0 $ such that $f_{\rho}^{*} \in [\mathcal{H}]^{s}$ and
  \begin{displaymath}
    \| f_{\rho}^{*} \|_{[\mathcal{H}]^{s}} \le R.
  \end{displaymath}
\end{assumption}
Functions in $[\mathcal{H}]^{s}$ with smaller $s$ are less smooth, which will be harder for an algorithm to estimate.
\vspace{6pt}
\begin{assumption}[Moment of error]\label{ass mom of error}
  The noise $ \epsilon \coloneqq y - f_{\rho}^{*}(x)$ satisfies that there are constants $ \sigma, L > 0$ such that for any $ m \ge 2$,
  \begin{displaymath}
      \mathbb{E}\left(|\epsilon|^m \mid x\right) \leq \frac{1}{2} m ! \sigma^2 L^{m-2}, \quad \mu \text {-a.e. } x \in \mathcal{X}.
  \end{displaymath}
\end{assumption}
This is a standard assumption to control the noise such that the tail probability decays fast \citep{lin2020_OptimalConvergence,fischer2020_SobolevNorm}. It is satisfied for, for instance, the Gaussian noise with bounded variance or sub-Gaussian noise. Some literature (e.g., \citealt{steinwart2009_OptimalRates,PillaudVivien2018StatisticalOO,Jun2019KernelTR}, etc) also uses a stronger assumption $y \in [-L_{0},L_{0}]$ which directly implies both Assumption \ref{ass mom of error} and the boundedness of $f_{\rho}^{*}$.



\subsection{Review of state-of-the-art results}
For the convenience of comparing our results with previous works, we review state-of-the-art upper and lower bounds of the convergence rate of KRR \cite{fischer2020_SobolevNorm}.
\vspace{6pt}
\begin{proposition}[Upper bound]\label{prop upper rate}
  Suppose that Assumption \ref{ass EDR},\ref{assumption embedding}, \ref{ass source condition} and \ref{ass mom of error} hold for $ 0 < s \le 2$ and $\frac{1}{\beta} \le \alpha_{0} < 1$. Furthermore, suppose that there exists a constant $B_{\infty} > 0$ such that $ \| f_{\rho}^{*} \|_{L^{\infty}(\mathcal{X},\mu)} \le B_{\infty}$. Let $\hat{f}_{\lambda}$ be the KRR estimator defined by \eqref{krr estimator}. Then in the case of $s + \frac{1}{\beta} > \alpha_{0} $, by choosing $\lambda \asymp n^{-\frac{\beta }{s \beta + 1}}$, for any fixed $\delta \in (0,1)$, when $n$ is sufficiently large, with probability at least $1 - \delta$, we have
      \begin{displaymath}
          \left\|\hat{f}_{\lambda}-f_{\rho}^*\right\|_{L_2}^2 \leq\left(\ln \frac{4}{\delta}\right)^2 C n^{-\frac{s \beta}{s \beta+1}},
      \end{displaymath}
      where $C$ is a constant independent of $n$ and $\delta$.
\end{proposition}
\vspace{6pt}
\begin{remark}
  When $s + \frac{1}{\beta} \le \alpha_{0} $, \citet[Theorem 1]{fischer2020_SobolevNorm} also proves the state-of-the-art upper bound of the convergence rate $ (n/\log^{r}(n))^{-\frac{s}{\alpha}}, \forall r>0$ and $\alpha > \alpha_{0}$ which can be arbitrarily close to $\alpha_{0}$. We should note that $ s + \frac{1}{\beta} > \alpha_{0} $ is always satisfied for Sobolev RKHS (see Section \ref{section sobolev}), so we focus on $ s + \frac{1}{\beta} > \alpha_{0} $ in this paper. 
\end{remark}
      

\vspace{6pt}
\begin{proposition}[Minimax lower bound]\label{prop lower rate}
   Let $\mu$ be a probability distribution on $\mathcal{X}$ such that Assumption \ref{ass EDR} and \ref{assumption embedding} are satisfied for $\frac{1}{\beta} \le \alpha_{0} < 1$. Let $\mathcal{P}$ consist of all the distributions on $\mathcal{X}\times \mathcal{Y}$ satisfying \ref{ass source condition}, \ref{ass mom of error} for $ 0 < s \le 2$ and with marginal distribution $\mu$. For a constant $B_{\infty} >0$, let $ \mathcal{P}_{\infty}$ consists of all the distributions on $\mathcal{X}\times \mathcal{Y}$ such that $ \| f_{\rho}^{*} \|_{L^{\infty}(\mathcal{X},\mu)} \le B_{\infty} $. Then for any $\alpha > \alpha_{0}$, there exists a constant $C$, for all learning methods, for any fixed $\delta \in (0,1)$, when $n$ is sufficiently large, there is a distribution $\rho \in \mathcal{P} \cap \mathcal{P}_{\infty}$ such that, with probability at least $1 - \delta$, we have 
   \begin{displaymath}
       \left\|\hat{f}-f_\rho^*\right\|_{L^2}^2 \ge C \delta n^{-\frac{\max\{s,\alpha\}  \beta}{\max\{s,\alpha\} \beta +1}}.
   \end{displaymath}
\end{proposition}
\vspace{6pt}
\begin{remark}
  Under the precondition $ s + \frac{1}{\beta} > \alpha_{0}$, (1) when $ s > \alpha_{0} $, since there always exists $\alpha_{0} < \alpha \le s$, the upper bound of the convergence rate in Proposition \ref{prop upper rate} coincides with the minimax lower bound in Proposition \ref{prop lower rate} and hence is minimax optimal; (2) but when $ s \le \alpha_{0} $, existing results fail to show the optimality of KRR. The same gap between the upper and lower bound exists for gradient descent and stochastic gradient descent with multiple passes \cite{PillaudVivien2018StatisticalOO}.
\end{remark}

Note that in Proposition \ref{prop lower rate}, we consider the distributions in $\mathcal{P} \cap \mathcal{P}_{\infty} $. If we consider all the distributions in $\mathcal{P}$, we have the following minimax lower bound, which is often referred to as the information-theoretic lower bound (see, e.g., \citealt{rastogi2017_OptimalRates}).
\vspace{6pt}
\begin{proposition}[Information-theoretic lower bound]\label{prop information lower bound}
    Let $\mu$ be a probability distribution on $\mathcal{X}$ such that Assumption \ref{ass EDR} is satisfied. Let $\mathcal{P}$ consist of all the distributions on $\mathcal{X}\times \mathcal{Y}$ satisfying \ref{ass source condition}, \ref{ass mom of error} for $ 0 < s \le 2$ and with marginal distribution $\mu$. Then there exists a constant $C$, for all learning methods, for any fixed $\delta \in (0,1)$, when $n$ is sufficiently large, there is a distribution $\rho \in \mathcal{P}$ such that, with probability at least $1 - \delta$, we have 
   \begin{displaymath}
       \left\|\hat{f}-f_\rho^*\right\|_{L^2}^2 \ge C \delta n^{-\frac{s  \beta}{s \beta +1}}.
   \end{displaymath}
\end{proposition}


\section{Main Results}
The main results of this paper aim to remove the boundedness assumption $ \| f_{\rho}^{*} \|_{L^{\infty}(\mathcal{X},\mu)} \le B_{\infty}$. We state the main theorem in terms of general RKHS, and make detailed discussions for Sobolev space as a particular case. 
\vspace{6pt}
\begin{theorem}\label{main theorem}
  Suppose that Assumption \ref{ass EDR}, \ref{assumption embedding}, \ref{ass source condition} and \ref{ass mom of error} hold for $ 0 < s \le 2$ and $\frac{1}{\beta} \le \alpha_{0} < 1$. Suppose that $f_{\rho}^{*} \in L^{q}(\mathcal{X},\mu)$ and $ \|f_{\rho}^{*} \|_{L^{q}(\mathcal{X},\mu)} \le C_{q} < \infty $ for some $q > \frac{2(s \beta + 1)}{2 + (s-\alpha_{\tiny 0}) \beta}$. Let $\hat{f}_{\lambda}$ be the KRR estimator defined by \eqref{krr estimator}. Then, in the case of $ s +\frac{1}{\beta} > \alpha_{0}$, by choosing $\lambda \asymp n^{-\frac{\beta }{s \beta + 1}}$, for any fixed $\delta \in (0,1)$, when $n$ is sufficiently large, with probability at least $1 - \delta$, we have
  \begin{displaymath}
      \left\|\hat{f}_{\lambda}-f_{\rho}^{*}\right\|_{L^{2}}^2 \leq\left(\ln \frac{4}{\delta}\right)^2 C n^{-\frac{s \beta}{s \beta +1}} ,
  \end{displaymath}
  where $C$ is a constant independent of $n$ and $\delta$.
\end{theorem}
\vspace{6pt}
\begin{remark}
  In Theorem \ref{main theorem}, we replace the boundedness assumption with a $L^{q}$-integrability assumption and prove the same upper bound of the convergence rate as Proposition \ref{prop upper rate} in the case of $ s +\frac{1}{\beta} > \alpha_{0}$. As shown in the following, both $s +\frac{1}{\beta} > \alpha_{0}$ and $q > \frac{2(s \beta + 1)}{2 + (s-\alpha_{\tiny 0}) \beta}$ are naturally satisfied for Sobolev RKHS. 
\end{remark}
\vspace{6pt}
\begin{remark}\label{remark of ui RKHS}
  RKHS with uniformly bounded eigenfunctions, i.e., $\sup_{i \in N}  \| e_{i} \|_{L^{\infty}} <\infty$ are frequently considered \citep{mendelson2010_RegularizationKernel, steinwart2009_OptimalRates, PillaudVivien2018StatisticalOO}. For this kind of RKHS, the Assumption \ref{assumption embedding} holds for $\alpha_{0} = \frac{1}{\beta}$ \citep[Lemma 10]{fischer2020_SobolevNorm}, hence $s +\frac{1}{\beta} > \alpha_{0}$ is satisfied. In addition, the assumption $q > \frac{2(s \beta + 1)}{2 + (s-\alpha_{\tiny 0}) \beta} $ in Theorem \ref{main theorem} turns into $q > 2$. Recalling that $ [\mathcal{H}]^{0} \subseteq L^{2}(\mathcal{X},\mu)$ when $s=0$, so it is reasonable to assume that the functions in $[\mathcal{H}]^{s}, s>0, $ is $L^{q}$-integrable for some $q >2$.
\end{remark}
\vspace{6pt}
\begin{remark}
  If the RKHS $\mathcal{H}$ satisfies that $[\mathcal{H}]^{s} \hookrightarrow L^{q}$, i.e., $[\mathcal{H}]^{s} \cap L^{q} = [\mathcal{H}]^{s}$ for some $q$ satisfying the integrability required in Theorem \ref{main theorem}. Using Proposition \ref{prop information lower bound}, the minimax lower bound will be still $ n^{-\frac{s \beta}{s \beta + 1}}$ even when making the assumption $ \|f_{\rho}^{*} \|_{L^{q}} < \infty $.
\end{remark}

\subsection{Optimality for Sobolev RKHS}\label{section sobolev}
Let us first introduce some concepts of (fractional) Sobolev space (see, e.g., \citealt{adams2003_SobolevSpaces}). In this section, we assume that $\mathcal{X} \subseteq \mathbb{R}^{d}$ is a bounded domain with smooth boundary and the Lebesgue measure $\nu$. Denote $L^{2}(\mathcal{X}) \coloneqq L^{2}(\mathcal{X},\nu)$ as the corresponding $L^{2}$ space. For $m \in \mathbb{N}$, we denote $H^{m}(\mathcal{X})$ as the Sobolev space with smoothness $m$ and $H^{0}(\mathcal{X}) \coloneqq L^{2}(\mathcal{X})$. Then the (fractional) Sobolev space for any real number $r >0 $ can be defined through \textit{real interpolation}:
\begin{displaymath}
      H^{r}(\mathcal{X}) := \left(L^{2}(\mathcal{X}), H^{m}(\mathcal{X})\right)_{\frac{r}{m},2},
  \end{displaymath}
where $m:=\min \{k \in \mathbb{N}: k > r\}$. (For details of real interpolation of Banach spaces, we refer to \citet[Chapter 4.2.2]{sawano2018theory}). It is well known that when $r > \frac{d}{2}$, $H^{r}(\mathcal{X})$ is a separable RKHS with respect to a bounded kernel, and the corresponding EDR is (see, e.g., \citealt{edmunds_triebel_1996}) 
\begin{displaymath}
  \beta = \frac{2 r}{d}.
\end{displaymath}
Furthermore, for the interpolation space of $H^{r}(\mathcal{X}) $ under Lebesgue measure defined by \eqref{def interpolation space}, we have (see, e.g., \citealt[Theorem 4.6]{steinwart2012_MercerTheorem}), for $ s > 0$,
\begin{displaymath}
  [H^{r}(\mathcal{X})]^{s} = H^{rs}(\mathcal{X}).
\end{displaymath}
Now we begin to introduce the embedding theorem of (fractional) Sobolev space from 7.57 of \citet{adams1975sobolev}, which is crucial when applying Theorem \ref{main theorem} to Sobolev RKHS. 
\vspace{6pt}
\begin{proposition}[Embedding theorem]\label{prop embedding theorem}
   Let $H^{r}(\mathcal{X}), r > 0$ be defined as above, we have
  \begin{enumerate}[(\romannumeral1)]
      
     \item if $d > 2r$, then $ H^{r}(\mathcal{X}) \hookrightarrow L^{q}(\mathcal{X})$ for $2 \le q \le \frac{2d}{d - 2r} $. 
     \item if $d = 2r$, then $ H^{r}(\mathcal{X}) \hookrightarrow L^{q}(\mathcal{X})$ for $2 \le q < \infty $.
     \item if $ d < 2(r-j)$ for some nonnegative integer $j$, then $ H^{r}(\mathcal{X}) \hookrightarrow C^{j,\gamma}(\mathcal{X}), \gamma= r-j-\frac{d}{2}$,
  \end{enumerate}
  where $C^{j,\gamma}(\mathcal{X}) $ denotes the Hölder space and $\hookrightarrow $ denotes the continuous embedding.
\end{proposition}
\vspace{6pt}
On the one hand, (\romannumeral3) of Proposition \ref{prop embedding theorem} shows that, for a Sobolev RKHS $ \mathcal{H} = H^{r}(\mathcal{X}), r > \frac{d}{2}$ and any $\alpha > \frac{1}{\beta} = \frac{d}{2r}$,
\begin{displaymath}
    [H^{r}(\mathcal{X})]^{\alpha} = H^{r\alpha}(\mathcal{X}) \hookrightarrow C^{0,\gamma}(\mathcal{X}) \hookrightarrow L^{\infty}(\mathcal{X}),
\end{displaymath}
where $\gamma > 0$. So the Assumption \ref{assumption embedding} holds for $\alpha_{0} = \frac{1}{\beta} $, and thus $\frac{2(s \beta + 1)}{2 + (s-\alpha_{0}) \beta} = 2$. On the other hand, (\romannumeral1) of Proposition \ref{prop embedding theorem} shows that if $ d > 2rs $,
\begin{displaymath}
    [H^{r}(\mathcal{X})]^{s} = H^{rs}(\mathcal{X}) \hookrightarrow L^{q}(\mathcal{X}),
\end{displaymath}
where $ q = \frac{2d}{d-2rs} = \frac{2}{1 - s \beta} > 2$. In addition, (\romannumeral2) and (\romannumeral3) of Proposition \ref{prop embedding theorem} show that $q>2$ also holds if $ d \le 2rs $. Therefore, for any $0 < s \le 2$, we have
\begin{displaymath}
    \text{Assumption} ~\ref{assumption embedding};~ s+ \frac{1}{\beta} > \alpha_{0};~ q > \frac{2(s \beta + 1)}{2 + (s-\alpha_{0}) \beta} = 2
\end{displaymath}
hold simultaneously.

Now we are ready to state a theorem as the corollary of Theorem \ref{main theorem} and Proposition \ref{prop embedding theorem}, which implies the optimality of KRR for Sobolev RKHS and any $0 < s \le 2$.
\vspace{6pt}
\begin{theorem}\label{upper rate of Sobolev}
  Let $\mathcal{X} \subseteq \mathbb{R}^{d}$ be a bounded domain with a smooth boundary. The RKHS is $\mathcal{H} = H^{r}(\mathcal{X})$ for some $r > d/2 $. Suppose that the distribution $\rho$ satisfies that $ \| f_{\rho}^{*}\|_{[\mathcal{H}]^{s}} \le R $ for $ 0 < s \le 2$, the noise satisfies Assumption \ref{ass mom of error}, and the marginal distribution $\mu$ on $\mathcal{X}$ has Lebesgue density $ 0 < c \le p(x) \le C$ for two constants $c$ and $C$. Let $\hat{f}_{\lambda}$ be the KRR estimator defined by \eqref{krr estimator}. Then, by choosing $\lambda \asymp n^{-\frac{\beta }{s \beta + 1}}$, for any fixed $\delta \in (0,1)$, when $n$ is sufficiently large, with probability at least $1 - \delta$, we have
  \begin{displaymath}
      \left\|\hat{f}_{\lambda}-f_{\rho}^{*}\right\|_{L^{2}}^2 \leq\left(\ln \frac{4}{\delta}\right)^2 C n^{-\frac{s \beta}{s \beta +1}} ,
  \end{displaymath}
  where $C$ is a constant independent of $n$ and $\delta$.
\end{theorem}
Note that we say that the distribution $\mu$ has Lebesgue density $0 < c \le p(x) \le C $, if $\mu$ is equivalent to the Lebesgue measure $\nu$, i.e., $ \mu \ll \nu, \nu \ll \mu$, and there exist constants $c, C > 0$ such that $c \leq \frac{\mathrm{d} \mu}{\mathrm{d} \nu} \leq C$.
\vspace{6pt}
\begin{remark}[Optimality for Sobolev RKHS]
  Without the boundedness assumption, Theorem \ref{upper rate of Sobolev} proves the same upper bound of the convergence rate $n^{-\frac{s \beta}{s \beta +1}}$. Together with the information-theoretic lower bound in Proposition \ref{prop information lower bound}, we prove the optimality of KRR for all $0<s\le 2$, while state-of-the-art result \citet[Corollary 5]{fischer2020_SobolevNorm} can only prove for $ \frac{1}{\beta} < s \le 2$. Since when $s > 2 $, the saturation effect of KRR has been proved in a recent work \citet{li2023saturation}, the optimality of KRR for Sobolev spaces is well understood. 
\end{remark}

\subsection{Sketch of proof}
In this subsection, we present the sketch of the proof of Theorem \ref{main theorem}. The rigorous proof of Theorem \ref{main theorem}, Proposition \ref{prop information lower bound}, and Theorem \ref{upper rate of Sobolev} will be in the appendix. We refer to \citet[Chapter 6]{fischer2020_SobolevNorm} for the proof of Proposition \ref{prop upper rate} and \ref{prop lower rate}.

Our proofs are based on the standard integral operator techniques dating back to \citet{Smale2007LearningTE}, and we refine the proof to handle the unbounded case. Let us first introduce some frequently used notations. Define the sample operator as
\begin{displaymath}
    K_{x}: \mathbb{R} \rightarrow \mathcal{H}, ~~ y \mapsto y k(x,\cdot),
\end{displaymath}
and its adjoint operator
\begin{displaymath}
  K_{x}^{*}: \mathcal{H} \rightarrow \mathbb{R},~~ f \mapsto f(x).
\end{displaymath}
Next we define the sample covariance operator $T_{X}: \mathcal{H} \to \mathcal{H}$ as
\begin{equation}\label{def of TX}
    T_X:=\frac{1}{n} \sum_{i=1}^n K_{x_i} K_{x_i}^*,
\end{equation}
and the sample basis function 
\begin{displaymath}
  g_Z:=\frac{1}{n} \sum_{i=1}^n K_{x_i} y_i \in \mathcal{H}.
\end{displaymath}
Then the KRR estimator \eqref{krr estimator} can be expressed by (see, e.g., \citet[Chapter 5]{Caponnetto2007OptimalRF}) 
\begin{displaymath}
  \hat{f}_\lambda = \left(T_X + \lambda\right)^{-1} g_Z .
\end{displaymath}
Define $f_{\lambda}$ as the unique minimizer given by
\begin{displaymath}
    f_\lambda = \underset{f \in \mathcal{H}}{\arg \min } \left( \int_{\mathcal{X}\times\mathcal{Y}} \left(f(x) - y\right)^{2} \mathrm{d}\rho(x,y) + \lambda\|f\|_{\mathcal{H}}^2\right).
\end{displaymath}
Note that the integral operator $T$ can also be seen as a bounded linear operator on $\mathcal{H}$, $ f_{\lambda} $ can be expressed by
\begin{equation}\label{expression of flambda}
    f_\lambda = \left(T+ \lambda\right)^{-1} g,
\end{equation}
where $g$ is the expectation of $g_{Z}$ given by
\begin{displaymath}
    g = \mathbb{E} g_{Z} = \int_{\mathcal{X}} k(x,\cdot) f_{\rho}^{*}(x) d\mu(x) = T f_{\rho}^{*} \in \mathcal{H}.
\end{displaymath}

The first step in our proof is to decompose the generalization error into two terms, which are often referred to as the approximation error and estimation error,
\begin{equation}\label{error decompose}
    \left\|\hat{f}_{\lambda}-f_{\rho}^{*}\right\|_{L^{2}} \le \left\|\hat{f}_{\lambda}-f_{\lambda}\right\|_{L^{2}} + \left\|f_{\lambda}-f_{\rho}^{*}\right\|_{L^{2}},
\end{equation}
Then we will show that by choosing $\lambda \asymp n^{-\frac{\beta}{s \beta + 1}} $,
\begin{equation}\label{approximation error}
    \left\|f_{\lambda}-f_{\rho}^{*}\right\|_{L^{2}} \le C R \lambda^{\frac{s}{2}} = C R n^{-\frac{1}{2} \frac{s \beta}{s \beta +1}};
\end{equation}
and for any fixed $\delta \in (0,1)$, when $n$ is sufficiently large, with probability at least $ 1- \delta$, we have
\begin{equation}\label{estimation error}
    \left\|\hat{f}_\lambda-f_{\lambda}\right\|_{L^{2}} \le \ln \frac{4}{\delta} C \frac{\lambda^{-\frac{1}{2\beta}}}{\sqrt{n}} = \ln \frac{4}{\delta} C n^{-\frac{1}{2}\frac{s \beta}{s \beta +1}}.
\end{equation}
Plugging \eqref{approximation error} and \eqref{estimation error} into \eqref{error decompose} and we will finish the proof of Theorem \ref{main theorem}.

\paragraph{The approximation error}
Suppose that $ f_{\rho}^{*} = \sum_{i \in N} a_{i} e_{i}$. Then using the expression \eqref{expression of flambda} and simple inequality, we will show that
\begin{displaymath}
    \left\|f_{\lambda}-f_{\rho}^{*}\right\|_{L^{2}}^{2} = \sum\limits_{i \in N} \left(\frac{\lambda \lambda_{i}^{s/2}}{\lambda + \lambda_{i}} \right)^{2} \lambda_{i}^{-s} a_{i}^{2} \le \lambda^{s} \|f_{\rho}^{*}\|_{[\mathcal{H}]^{s}}.
\end{displaymath}

\paragraph{The estimation error}
Denote $ T_{X \lambda} = T_{X} + \lambda$ and $ T_{\lambda} = T + \lambda$.  We will first rewrite the estimator error as follows
\begin{align}\label{sketch-0}
  &\left\|\hat{f}_\lambda - f_{\lambda}\right\|_{L^{2}} = \left\| T^{\frac{1}{2}} \left( \hat{f}_\lambda-f_{\lambda} \right)\right\|_{\mathcal{H}} \notag \\
  &\le \left\| T^{\frac{1}{2}} T_{\lambda}^{-\frac{1}{2}} \right\| \cdot \left\| T_{\lambda}^{\frac{1}{2}} T_{X \lambda}^{-1} T_{\lambda}^{\frac{1}{2}} \right\| \cdot \left\| T_{\lambda}^{-\frac{1}{2}} \left( g_{Z} - T_{X \lambda} f_{\lambda} \right) \right\|_{\mathcal{H}}.
\end{align}
For the first term in \eqref{sketch-0}, we have
\begin{displaymath}
    \left\| T^{\frac{1}{2}} T_{\lambda}^{-\frac{1}{2}} \right\| = \sup\limits_{i \in N} \left(\frac{\lambda_{i}}{\lambda_{i} + \lambda}\right)^{\frac{1}{2}} \le 1.
\end{displaymath}
For the second term in \eqref{sketch-0}, using the concentration result between $ T_{X \lambda} $ and $T_{\lambda} $, it can also be bounded by a constant. The main challenge of the proof is bounding the third term. We will show that the third term in \eqref{sketch-0} can be rewritten as 
\begin{align}\label{sketch-1}
    &\left\|T_\lambda^{-\frac{1}{2}}\left[\left(g_Z - \left(T_X + \lambda + T - T \right) f_\lambda\right)\right]\right\|_{\mathcal{H}} \notag \\
    &=\left\|T_\lambda^{-\frac{1}{2}}\left[\left(g_Z - T_X f_\lambda\right) - \left(T + \lambda \right) f_\lambda + T f_\lambda \right]\right\|_{\mathcal{H}} \notag \\
    &= \left\|T_\lambda^{-\frac{1}{2}}\left[\left(g_Z-T_X f_\lambda\right)-\left(g-T f_\lambda\right)\right]\right\|_{\mathcal{H}}.
\end{align}
The form of \eqref{sketch-1} allows us to use the well-known Bernstein type inequality, which controls the difference between a sum of i.i.d. random variables and its expectation. Traditionally, this step requires the boundedness of $f_{\rho}^{*}$. We refine this procedure by a truncation method and prove that with high probability, 
\begin{displaymath}
 \eqref{sketch-1} \le \ln{\frac{2}{\delta}} C \frac{\lambda^{-\frac{1}{2 \beta}}}{\sqrt{n}} = \ln{\frac{2}{\delta}} C n^{-\frac{1}{2} \frac{s \beta}{s \beta +1}}.   
\end{displaymath}
Specifically, we consider two subset of $\mathcal{X}$: $\Omega_{1} = \{x \in \Omega: |f_{\rho}^{*}(x)| \le t \}$ and $\Omega_{2} = \mathcal{X} \backslash \Omega_{1}$ where $t$ is allowed to diverge as $n \to \infty$. When choosing $t$ as a proper order of $n$, we show that on the one hand, the norm in $\Omega_{1}$ is still upper bounded by $ n^{-\frac{1}{2} \frac{s \beta}{s \beta +1}}$; on the other hand, the norms in $\Omega_{2}$ vanishes with a probability tending to 1 as $n \to \infty$. We argue that such $t$ exists by taking advantage of the $L^{q}$-integrability of $f_{\rho}^{*}$, which is required in the statement of Theorem \ref{main theorem}.

\section{Experiments}
In our experiments, we aim to verify that for functions $f_{\rho}^{*} \in [\mathcal{H}]^{s}$ but not in $L^{\infty}$, the KRR estimator can still achieve the convergence rate $n^{-\frac{s \beta}{ s \beta + 1}}$. We show the results for both Sobolev RKHS, and general RKHS with uniformly bounded eigenfunctions mentioned in Remark \ref{remark of ui RKHS}.
\subsection{Experiments in Sobolev space}\label{section experiment in sobolev}
Suppose that $\mathcal{X} = [0,1]$ and the marginal distribution $\mu$ is the uniform distribution on $[0,1]$. We consider the RKHS $\mathcal{H} = H^{1}(\mathcal{X})$ to be the Sobolev space with smoothness 1. Section \ref{section sobolev} shows that the EDR is $\beta = 2$ and embedding index is $\alpha_{0} = \frac{1}{\beta}$. We construct a function in $[\mathcal{H}]^{s} \backslash L^{\infty}$ by 
\begin{equation}\label{series of sobolev}
    f^{*}(x) =  \sum\limits_{k=1}^{\infty} \frac{1}{k^{s+0.5}} \left( \sin\left( 2 k \pi x\right) + \cos\left( 2 k \pi x\right) \right),
\end{equation}
for some $0 < s < \frac{1}{\beta} = 0.5$. We will show in Appendix \ref{appendix detail experiments} that the series in \eqref{series of sobolev} converges on $(0,1)$. In addition, since $ \sin 2 k \pi + \cos 2 k \pi \equiv 1$, we also have $f^{*} \notin L^{\infty}(\mathcal{X})$. The explicit formula of the kernel associated to $H^{1}(\mathcal{X})$ is given by \citet[Corollary 2]{ThomasAgnan1996ComputingAF}, i.e., $ k(x,y) = \frac{1}{\sinh{1}} \cosh{(1- \max(x,y)) \cosh{(1- \min(x,y))}}$. 

We consider the following data generation procedure:
\begin{displaymath}
    y = f^{*}(x) + \epsilon, 
\end{displaymath}
where $f^{*}$ is numerically approximated by the first 1000 terms in \eqref{series of sobolev} with $s=0.4$, $x \sim \mathcal{U}[0,1]$ and $\epsilon \sim \mathcal{N}(0,1)$. We choose the regularization parameter as $\lambda = c n^{-\frac{\beta}{s \beta + 1}} = c n^{-\frac{10}{9}}$ for a fixed $c$. The sample size $n$ is chosen from 1000 to 5000, with intervals of 100. We numerically compute the generalization error $ \|\hat{f} - f^{*}\|_{L^{2}}$ by Simpson's formula with $N \gg n$ testing points. For each $n$, we repeat the experiments 50 times and present the average generalization error as well as the region within one standard deviation. To visualize the convergence rate $r$, we perform logarithmic least-squares $\log \text{err} = r \log n + b$ to fit the generalization error with respect to the sample size and display the value of $r$. 

\begin{figure}[ht]
\vskip 0.1in
\begin{center}
\centerline{\includegraphics[width=\columnwidth]{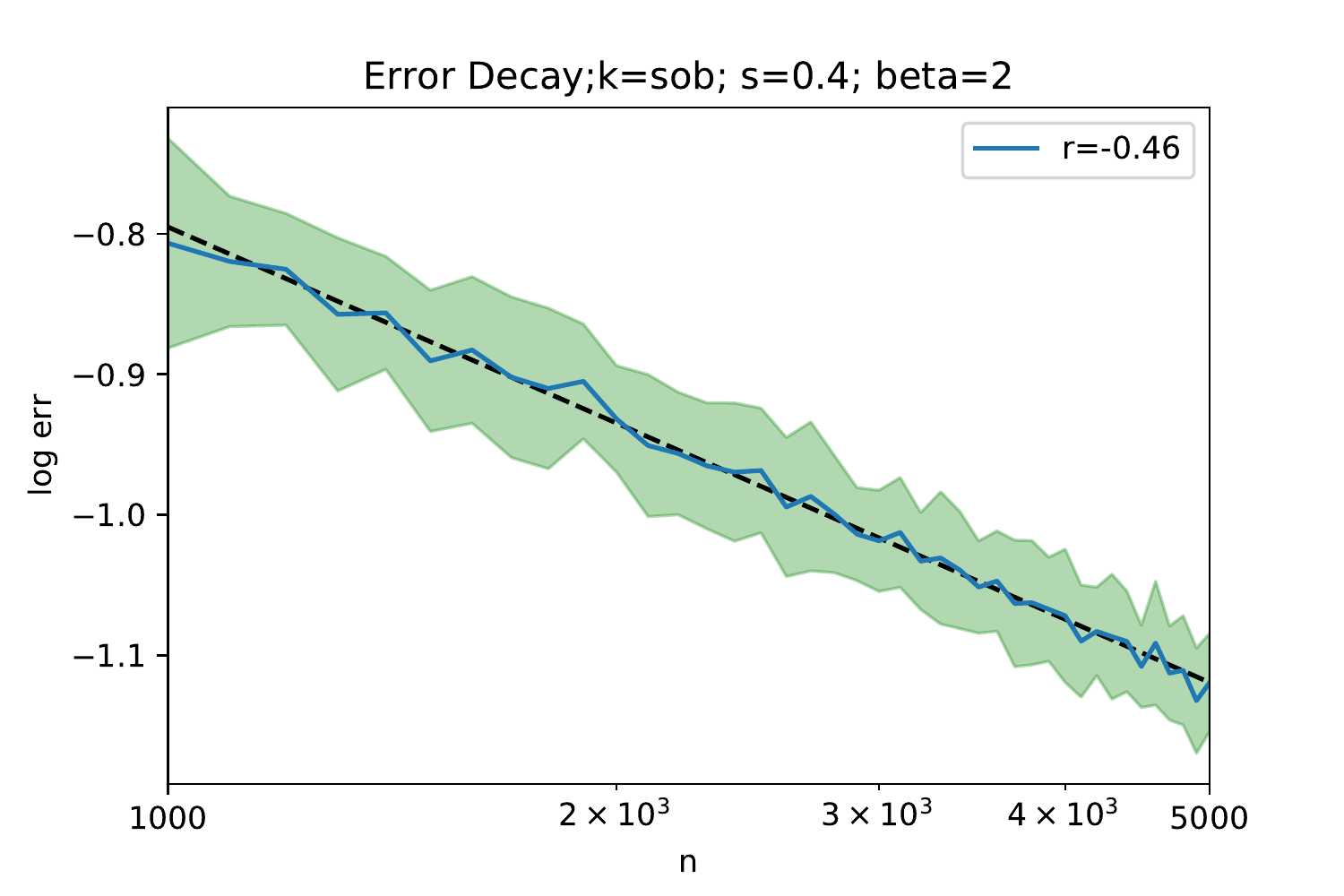}}
\caption{Error decay curve of Sobolev RKHS. Both axes are logarithmic. The curve shows the average generalization errors over 50 trials; the region within one standard deviation is shown in green.
The dashed black line is computed using logarithmic least-squares, and the slope represents the convergence rate $r$.}
\label{figure 0}
\end{center}
\vskip -0.2in
\end{figure}

We try different values of $c$, Figure \ref{figure 0} presents the convergence curve under the best choice of $c$. It can be concluded that the convergence rate of KRR's generalization error is indeed approximately equal to $ n^{-\frac{s \beta}{s \beta + 1}} = n^{-\frac{4}{9}}$, without the boundedness assumption of the true function $f^{*}$. We also did another experiment that used cross validation to choose the regularization parameter. Figure \ref{figure appendix sob_allc} in Appendix \ref{appendix detail experiments} shows a similar result as Figure \ref{figure 0}. We refer to Appendix \ref{appendix detail experiments} for more details of the experiments.
\vspace{6pt}
\begin{remark}
  This setting is similar to \citet{PillaudVivien2018StatisticalOO}. The difference is that we choose the source condition $s$ to be smaller than $\alpha_{0}$ so that $f^{*} \notin L^{\infty}(\mathcal{X})$.
\end{remark}

\subsection{Experiments in general RKHS}\label{section RKHS experiments}
Suppose that $\mathcal{X} = [0,1]$ and the marginal distribution $\mu$ is the uniform distribution on $[0,1]$. It is well known that the following RKHS 
\begin{align}
    \mathcal{H}:=\Big\{f:[0,1] \rightarrow \mathbb{R} \mid f &\text { is A.C., } f(0)=0, \notag \\ 
    &\int_0^1\left(f^{\prime}(x)\right)^2 \mathrm{~d} x<\infty\Big\}. \notag
\end{align}
is associated with the kernel $k(x,y) = \min(x,y)$ \cite{wainwright2019_HighdimensionalStatistics}. Further, its eigenvalues and eigenfunctions can be written as
\begin{displaymath}
    \lambda_n=\left(\frac{2 n-1}{2} \pi\right)^{-2}, \quad n=1,2, \cdots
\end{displaymath}
and
\begin{displaymath}
    e_n(x)=\sqrt{2} \sin \left(\frac{2 n-1}{2} \pi x\right), \quad n=1,2, \cdots
\end{displaymath}
It is easy to see that the EDR is $\beta = 2$, the eigenfunctions are uniformly bounded, and the embedding index is $\alpha_{0} = \frac{1}{\beta}$ (see Remark \ref{remark of ui RKHS}). We construct a function in $[\mathcal{H}]^{s} \backslash L^{\infty}$ by
\begin{equation}\label{series of min}
    f^{*}(x) = \sum\limits_{k = 1}^{\infty} \frac{1}{k^{s+0.5}} e_{2k-1}(x),
\end{equation}
for some $0< s < \frac{1}{\beta} = 0.5$. We will show in Appendix \ref{appendix detail experiments} that the series in \eqref{series of min} converges on $(0,1)$. Since $e_{2k-1}(1) \equiv 1$, we also have $f^{*} \notin L^{\infty}(\mathcal{X})$. 

We use the same data generation procedure as Section \ref{section experiment in sobolev}:
\begin{displaymath}
    y = f^{*}(x) + \epsilon, 
\end{displaymath}
where $f^{*}$ is numerically approximated by the first 1000 terms in \eqref{series of min} with $s=0.4$, $x \sim \mathcal{U}[0,1]$ and $\epsilon \sim \mathcal{N}(0,1)$.

Figure \ref{figure 1} presents the convergence curve under the best choice of $c$. It can also be concluded that the convergence rate of KRR's generalization error is indeed approximately equal to $ n^{-\frac{s \beta}{s \beta + 1}} = n^{-\frac{4}{9}}$. Figure \ref{figure appendix min_allc} in Appendix \ref{appendix detail experiments} shows the result of using cross validation. 
\begin{figure}[ht]
\vskip 0.1in
\begin{center}
\centerline{\includegraphics[width=\columnwidth]{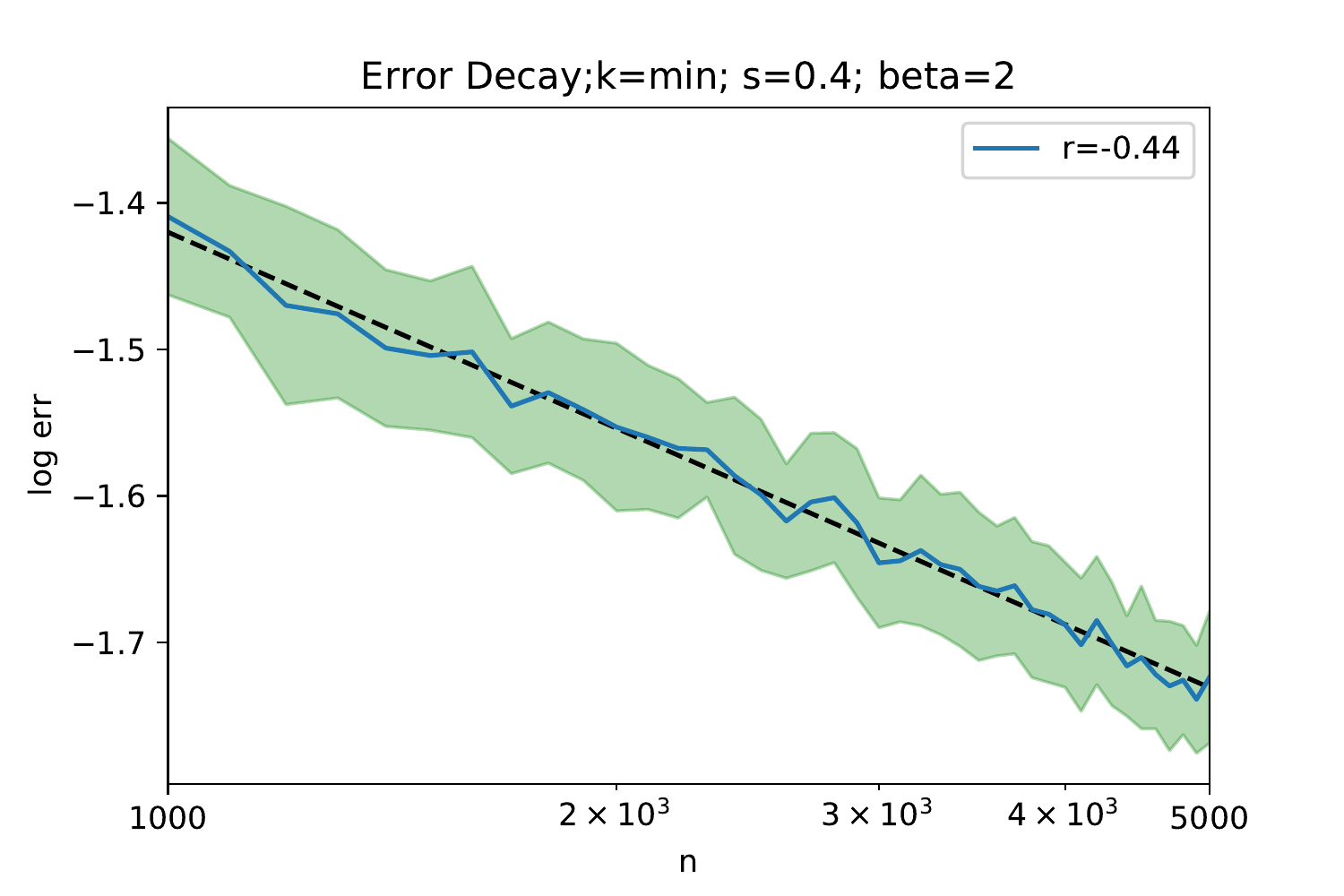}}
\caption{Error decay curve of general RKHS. Both axes are logarithmic. The curve shows the average generalization errors over 50 trials; the region within one standard deviation is shown in green.
The dashed black line is computed using logarithmic least-squares, and the slope represents the convergence rate $r$.}
\label{figure 1}
\end{center}
\vskip -0.2in
\end{figure}

\section{Conclusion and Discussion}
This paper considers the convergence rate and the minimax optimality of kernel ridge regression, especially in the misspecified case. When the true function $f_{\rho}^{*}$ falls into a less-smooth ($s \le \alpha_{0}$) interpolation space of the RKHS $ [\mathcal{H}]^{s} \nsubseteq L^{\infty}$, it has been an outstanding problem for years that whether the boundedness assumption of $f_{\rho}^{*}$ in the proof can be removed or weakened so that KRR is optimal for all $ 0 < s \le 2$. This paper proves that for Sobolev RKHS, the boundedness assumption can be directly removed. This result implies that, on the one hand, the desired upper bound of the convergence rate holds for all the functions in $[\mathcal{H}]^{s}$ (which can only be proved for $[\mathcal{H}]^{s} \cap L^{\infty}$ before); on the other hand, KRR is minimax optimal for all source conditions $0 < s \le 2$ (which can only be proved for $\frac{1}{\beta} < s \le 2$ before). Together with the saturation effect of KRR \citep{li2023saturation} when $s > 2$, all convergence behaviors of Sobolev RKHS are understood. For general RKHS, we also prove that the boundedness assumption can be replaced by a $L^{q}$-integrability assumption which turns out to be reasonable, at least for RKHS with uniformly bounded eigenfunctions. We verify the results through experiments for both Sobolev RKHS and general RKHS.

If an RKHS $\mathcal{H}$ has embedding index $\alpha_{0} = 1$, considering the embedding index will just recover the results of \cite{lin2018_OptimalRates}. Note that we assume $\alpha_{0} \in (0,1)$ throughout this paper, on which we will perform refined analysis. Technical tools in this paper can be extended to general spectral regularization algorithms and $[\mathcal{H}]^{\gamma}$-generalization error. Another direct open question is to discuss whether the $L^{q}$-integrability assumption in Theorem \ref{main theorem} holds for general RKHS, which we conjecture to be true. 

We also notice a line of work which studies the learning curves of kernel ridge regression \citep{Spigler2020AsymptoticLC,Bordelon2020SpectrumDL,Cui2021GeneralizationER} and crossovers between different noise magnitudes. At present, their results all rely on a Gaussian design assumption (or some variation), which is a very strong assumption. Nevertheless, their empirical results are enlightening and attract interest in the study of the learning curves of kernel ridge regression. We believe that studying the misspecified case in our paper is a crucial step to remove the Gaussian design assumption and draw complete conclusions about the learning curves of kernel ridge regression. KRR is also connected with Gaussian process regression~\citep{kanagawa2018_GaussianProcesses}.
\citet{Jin2021LearningCF} claimed to establish the learning curves for Gaussian process regression and thus for KRR\@.
However, there is a gap in their proof where an essential covering argument is missing:
their Corollary 20 provides a high probability bound for fixed $\nu$,
while the proof of their Lemma 40 and Lemma 41 mistakenly use Corollary 20 to assert that the bound holds simultaneously for all $\nu$.

Since \citet{Jacot_NTK_2018} introduced the NTK , kernel regression has become a natural surrogate for the neural networks in the `lazy trained regime'. Our work is also motivated by our empirical studies in the neural networks and NTK regressions.
Specifically, we found that the source condition with respect to the neural tangent kernel (or other frequently used kernels) is relatively small ($s \ll 1$) for some frequently used real datasets (MNIST, CIFAR-10). This observation urged us to determine the optimal convergence rate when the RKHS is misspecified.

\section*{Acknowledgements}
This research was partially supported by Beijing Natural Science Foundation (Grant Z190001), the National Natural Science Foundation of China (Grant 11971257),  National Key R\&D Program of China (2020AAA0105200), and Beijing Academy of Artificial Intelligence.




\bibliography{example_paper}
\bibliographystyle{icml2023}

\newpage
\appendix
\onecolumn

\section{Proof of Theorem \ref{main theorem}}
Throughout the proof, we denote 
\begin{equation}
   T_{\lambda} = T + \lambda;~~ T_{X \lambda} = T_{X} + \lambda,
\end{equation}
where $\lambda$ is the regularization parameter and $ T, T_X$ is defined by \eqref{def of T}, \eqref{def of TX}. We use $\| \cdot \|_{\mathscr{B}(B_1,B_2)}$ to denote the operator norm of a bounded linear operator from a Banach space $B_1$ to $B_2$, i.e., $ \| A \|_{\mathscr{B}(B_1,B_2)} = \sup\limits_{\|f\|_{B_1}=1} \|A f\|_{B_2} $. Without bringing ambiguity, we will briefly denote the operator norm as $\| \cdot \|$. In addition, we use $\text{tr}A$ and $\| A \|_{1}$ to denote the trace and the trace norm of an operator. We use $\| A \|_{2}$ to denote the Hilbert-Schmidt norm. In addition, we denote $ L^{2}(\mathcal{X},\mu)$ as $ L^{2}$, $ L^{\infty}(\mathcal{X},\mu)$ as $ L^{\infty}$ for brevity throughout the proof.

In addition, denote the effective dimension
\begin{align}
    \mathcal{N}(\lambda) = \text{tr}\big(T (T + \lambda)^{-1}\big) = \sum\limits_{i \in N} \frac{\lambda_{i}}{\lambda_{i} + \lambda}, 
\end{align}
since the EDR of $\mathcal{H}$ is $\beta$, Lemma \ref{lemma of effect} shows that $\mathcal{N}(\lambda) \asymp \lambda^{-\frac{1}{\beta}} $. 

The following lemma will be frequently used, so we present it at the beginning of our proof.
\begin{lemma}\label{lemma freq}
  For any $\lambda > 0$ and $ s \in [0,1]$, we have
  \begin{equation}
      \sup _{t \geq 0} \frac{t^s}{t+\lambda} \leq \lambda^{s-1}.
  \end{equation}
\end{lemma}
\begin{proof}
  Since $ a^{s} \le a + 1$ for any $a \ge 0$ and $s \in [0,1] $, the lemma follows from
  \begin{equation}
      \left(\frac{t}{\lambda}\right)^{s} \le \frac{t}{\lambda} + 1 = \frac{t + \lambda}{\lambda}.
  \end{equation}
\end{proof}

\subsection{Proof of the approximation error}

The following theorem gives the bound of $[\mathcal{H}]^{\gamma}$-norm of $ f_\lambda - f_{\rho}^*$ when $0 \le \gamma \le s $. As a special case, the approximation error $\left\|f_\lambda - f_{\rho}^* \right\|_{L^{2}} $ follows from the result when $\gamma = 0$. 
\begin{theorem}\label{theorem of approximation error}
   Suppose that Assumption \ref{ass source condition} holds for $ 0 < s \le 2$. Denote $f_{\lambda} = (T+\lambda)^{-1}g$, then for any $\lambda >0$ and $0 \le \gamma \le s $, we have 
   \begin{equation}\label{prop appr 1}
       \left\|f_\lambda - f_{\rho}^*\right\|_{[\mathcal{H}]^{\gamma}} \le R \lambda^{\frac{s-\gamma}{2}}.
   \end{equation}
\end{theorem}
\begin{proof}
    Suppose that $ f_{\rho}^{*} = \sum_{i \in N} a_{i} e_{i}$. Recall that $f_\lambda = \left(T+ \lambda\right)^{-1} g = \left(T+ \lambda\right)^{-1} T f_{\rho}^{*} $, we have
    \begin{align}
        \left\|f_{\lambda}-f_{\rho}^{*}\right\|_{[\mathcal{H}]^{\gamma}}^{2} &= \left\| \sum\limits_{i \in N} \frac{\lambda }{\lambda + \lambda_{i}}  a_{i} e_{i}(\cdot) \right\|_{[\mathcal{H}]^{\gamma}}^{2} = \sum\limits_{i \in N} \left(\frac{\lambda}{\lambda + \lambda_{i}} \right)^{2} \lambda_{i}^{-\gamma} a_{i}^{2} \notag \\
        &= \sum\limits_{i \in N} \left(\frac{\lambda \lambda_{i}^{\frac{s-\gamma}{2}}}{\lambda + \lambda_{i}} \right)^{2} \lambda_{i}^{-s} a_{i}^{2} \notag \\
        &\le \left(\lambda \sup\limits_{i \in N} \frac{ \lambda_{i}^{\frac{s-\gamma}{2}}}{\lambda + \lambda_{i}} \right)^{2} \sum\limits_{i \in N} \lambda_{i}^{-s} a_{i}^{2} \notag \\
        &\le \lambda^{\frac{s-\gamma}{2}}  \|f_{\rho}^{*}\|_{[\mathcal{H}]^{s}},  \notag \\
    \end{align}
    where we use Lemma \ref{lemma freq} for the last inequality and the $[\mathcal{H}]^{s} $ norm defined by \eqref{def of interpolation norm}. Then the theorem follows from Assumption \ref{ass source condition}, i.e., $\| f_{\rho}^{*} \|_{[\mathcal{H}]^{s}} \le R. $

\end{proof}

\subsection{Proof of the estimation error}
\begin{theorem}\label{estimation error thm}
  Suppose that Assumption \ref{ass EDR}, \ref{assumption embedding}, \ref{ass source condition} and \ref{ass mom of error} hold for $ 0 < s \le 2$ and $\frac{1}{\beta} \le \alpha_{0} < 1$. Suppose that $f_{\rho}^{*} \in L^{q}(\mathcal{X},\mu)$ and $ \|f_{\rho}^{*} \|_{L^{q}(\mathcal{X},\mu)} \le C_{q} < \infty $ for some $q > \frac{2(s \beta + 1)}{2 + (s-\alpha_{\tiny 0}) \beta}$. Then in the case of $s + \frac{1}{\beta} > \alpha_{0}$, by choosing $ \lambda \asymp n^{-\frac{ \beta}{s \beta + 1}}$, for any fixed $\delta \in (0,1)$, when $n$ is sufficiently large, with probability at least $ 1- \delta$ , we have 
  \begin{equation}
      \left\|\hat{f}_\lambda-f_{\lambda}\right\|_{L^{2}} \le \ln \frac{4}{\delta} C n^{-\frac{1}{2}\frac{s \beta}{s \beta +1}}.
  \end{equation}
  where $C$ is a constant that only depends on $ \kappa, R, L, \sigma, C_{q} $.
\end{theorem}
\begin{proof}
   First, rewrite the estimator error as follows
\begin{align}\label{appr proof-0}
  &\left\|\hat{f}_\lambda - f_{\lambda}\right\|_{L^{2}} = \left\| T^{\frac{1}{2}} \left( \hat{f}_\lambda-f_{\lambda} \right)\right\|_{\mathcal{H}} \le \left\| T^{\frac{1}{2}} T_{\lambda}^{-\frac{1}{2}} \right\| \cdot \left\| T_{\lambda}^{\frac{1}{2}} T_{X \lambda}^{-1} T_{\lambda}^{\frac{1}{2}} \right\| \cdot \left\| T_{\lambda}^{-\frac{1}{2}} \left( g_{Z} - T_{X \lambda} f_{\lambda} \right) \right\|_{\mathcal{H}}.
\end{align}
For the first term in \eqref{appr proof-0}, we have
\begin{equation}\label{appr proof term1}
    \left\| T^{\frac{1}{2}} T_{\lambda}^{-\frac{1}{2}} \right\| = \sup\limits_{i \in N} \left(\frac{\lambda_{i}}{\lambda_{i} + \lambda}\right)^{\frac{1}{2}} \le 1.
\end{equation}
For the second term in \eqref{appr proof-0}, for any $\alpha > \alpha_{0}$, using Proposition \ref{lemma4.6} and we known that when $\lambda,n$ satisfy that
\begin{equation}\label{require condition}
    u \coloneqq \frac{M_{\alpha}^{2} \lambda^{-\alpha}}{n} \ln \frac{4 \kappa^2 \mathcal{N}(\lambda)\left(\|T\|+\lambda \right)}{\frac{\delta}{2}\|T\|} \leq \frac{1}{8},
\end{equation}
we have
\begin{equation}
  a \coloneqq \Vert T_\lambda^{-\frac{1}{2}} (T - T_X) T_\lambda^{-\frac{1}{2}} \Vert \le \frac{4}{3} u + \sqrt{2 u} \le \frac{2}{3}.
\end{equation}
with probability as least $1-\frac{\delta}{2}$. Therefore, 
\begin{align}\label{appr proof term2}
    \left\| T_{\lambda}^{\frac{1}{2}} T_{X \lambda}^{-1} T_{\lambda}^{\frac{1}{2}} \right\| &= \left\|\left( T_{\lambda}^{-\frac{1}{2}} (T_{X}+\lambda) T_{\lambda}^{-\frac{1}{2}} \right)^{-1}\right\| \notag \\
    &= \left\|\left(I - T_{\lambda}^{-\frac{1}{2}} (T_{X}-T) T_{\lambda}^{-\frac{1}{2}} \right)^{-1}\right\| \notag \\
    &\le \sum\limits_{k=0}^{\infty} \left\| T_{\lambda}^{-\frac{1}{2}} (T_{X}-T) T_{\lambda}^{-\frac{1}{2}} \right\|^{k}  \notag \\
    &\le \sum\limits_{k=0}^{\infty} \left(\frac{2}{3}\right)^{k} \le 3,
\end{align}
with probability as least $1-\frac{\delta}{2}$, where $I$ is the identity operator $\mathcal{H} \to \mathcal{H}$. Note that since we have $s +\frac{1}{\beta} > \alpha_{0} $, there always exists $\alpha_{0} < \alpha < s +\frac{1}{\beta}$ such that \eqref{require condition} is satisfied when $\lambda \asymp n^{-\frac{ \beta}{s \beta + 1}} $ and $n$ is sufficiently large . 

For the third term in \eqref{appr proof-0}, it can be rewritten as 
\begin{align}\label{appr proof-1}
    \left\| T_{\lambda}^{-\frac{1}{2}} \left( g_{Z} - T_{X \lambda} f_{\lambda} \right) \right\|_{\mathcal{H}} &=\left\|T_\lambda^{-\frac{1}{2}}\left[\left(g_Z - \left(T_X + \lambda + T - T \right) f_\lambda\right)\right]\right\|_{\mathcal{H}} \notag \\
    &=\left\|T_\lambda^{-\frac{1}{2}}\left[\left(g_Z - T_X f_\lambda\right) - \left(T + \lambda \right) f_\lambda + T f_\lambda \right]\right\|_{\mathcal{H}} \notag \\
    &= \left\|T_\lambda^{-\frac{1}{2}}\left[\left(g_Z-T_X f_\lambda\right)-\left(g-T f_\lambda\right)\right]\right\|_{\mathcal{H}}.
\end{align}
Using Proposition \ref{theorem 4.9},with probability at least $1-\frac{\delta}{2}$, we have 
\begin{equation}\label{appr proof term3}
      \left\|T_\lambda^{-\frac{1}{2}}\left[\left(g_Z-T_X f_\lambda\right)-\left(g-T f_\lambda\right)\right]\right\|_{\mathcal{H}} \le \ln{\frac{4}{\delta}} C \frac{\lambda^{-\frac{1}{2 \beta}}}{\sqrt{n}} = \ln{\frac{4}{\delta}} C n^{-\frac{1}{2} \frac{s \beta}{s \beta +1}}.
\end{equation}
Plugging \eqref{appr proof term1}, \eqref{appr proof term2} and \eqref{appr proof term3} into \eqref{appr proof-0} and we finish the proof.
\end{proof}

\subsection{Proof of Theorem \ref{main theorem}}
Using the approximation-estimation error decomposition,
\begin{equation}
    \left\|\hat{f}_{\lambda}-f_{\rho}^{*}\right\|_{L^{2}} \le \left\|\hat{f}_{\lambda}-f_{\lambda}\right\|_{L^{2}} + \left\|f_{\lambda}-f_{\rho}^{*}\right\|_{L^{2}},
\end{equation}
together with Theorem \ref{theorem of approximation error} and Theorem \ref{estimation error thm} and we finish the proof of Theorem \ref{main theorem}.

\section{Proof of Theorem \ref{upper rate of Sobolev}}
Note that the RKHS $\mathcal{H}$ is defined as the (fractional) Sobolev space $H^{r}(\mathcal{X})$, which is regardless of the marginal distribution $\mu$. But the definition of interpolation space \eqref{def interpolation space} is dependent on $\mu$. When $\mu$ has Lebesgue density $ 0 < c \le p(x) \le C$, \citet[(14)]{fischer2020_SobolevNorm} shows that 
\begin{equation}
    L^{2}(\mathcal{X},\mu) \cong  L^{2}(\mathcal{X},\nu),
\end{equation}
and
\begin{equation}
    \left[H^r(\mathcal{X})\right]_\mu^{s} \cong\left(L_2(\mathcal{X},\mu),\left[H^r(\mathcal{X})\right]^{1}_\mu\right)_{s, 2} \cong \left(L_2(\mathcal{X},\nu),\left[H^r(\mathcal{X})\right]^{1}_\nu\right)_{s, 2} \cong \left[H^r(\mathcal{X})\right]_{\nu}^{s} \cong H^{rs}(\mathcal{X}),
\end{equation}
where we denote $\left[H^r(\mathcal{X})\right]_\mu^{s} $ as the interpolation space of $H^r(\mathcal{X})$ under marginal distribution $\mu$. So we can also apply Proposition \ref{prop embedding theorem} to $\left[H^r(\mathcal{X})\right]_\mu^{s}$ and the embedding property of it is the same as $\left[H^r(\mathcal{X})\right]^{s}$. Denote $ \beta = \frac{2r}{d}$, since $\alpha_{0} = \frac{1}{\beta}$, for any $0 < s \le 2$, we have
\begin{equation}
    \text{Assumption} ~\ref{assumption embedding};~ s+ \frac{1}{\beta} > \alpha_{0};~ q > \frac{2(s \beta + 1)}{2 + (s-\alpha_{\tiny 0}) \beta}=2
\end{equation}
hold simultaneously.

Therefore, all the assumptions in Theorem \ref{main theorem} are satisfied, and the proof follows from applying Theorem \ref{main theorem}.

\section{Proof of Proposition \ref{prop information lower bound}}
We will construct a family of probability distributions on $ \mathcal{X} \times \mathcal{Y}$ and apply Lemma \ref{lower prop from tsy}. Recall that $\mu$ is a probability distribution on $\mathcal{X}$ such that Assumption \ref{ass EDR} is satisfied. Denote the class of functions 
\begin{equation}
    B^{s}(R)=\left\{f \in[\mathcal{H}]^s: \|f\|_{[\mathcal{H}]^{s}} \leq R\right\},
\end{equation}
and for every $f \in B^{s}(R)$, define the probability distribution $\rho_{f}$ on $\mathcal{X} \times \mathcal{Y}$ such that
\begin{equation}
    y = f(x) + \epsilon, ~~ x \sim \mu, 
\end{equation}
where $\epsilon \sim \mathcal{N}(0,\bar{\sigma}^{2})$ and $\bar{\sigma} = \min(\sigma, L) $. It is easy to show that such $\rho_{f}$ falls into the family $\mathcal{P}$ in Proposition \ref{prop information lower bound}. (Assumption \ref{ass EDR} and \ref{ass source condition} are satisfied obviously. Assumption \ref{ass mom of error} follows from results of moments of Gaussian random variables, see, e.g., \citet[Lemma 21]{fischer2020_SobolevNorm}).

Using Lemma \ref{lemma of ham}, for $m = n^{\frac{1}{s\beta + 1}}$, there exists $ \omega^{(0)}, \cdots, \omega^{(M)} \in \{0,1\}^{m}$ for some $M \ge 2^{m/8} $ such that
\begin{equation}\label{proof lower-2}
    \sum_{k=1}^m\left|\omega_k^{(i)}-\omega_k^{(j)}\right| \geq \frac{m}{8}, \quad \forall 0 \leq i<j \leq M.
\end{equation}
For $\epsilon = C_{0} m^{- s\beta - 1}$, define the functions $ f_i, i=1,2,\cdots, M $ as 
\begin{equation}
    f_i:=\epsilon^{1 / 2} \sum_{k=1}^m \omega_k^{(i)} e_{m+k}.
\end{equation}
Since
\begin{equation}\label{proof lower-1}
    \left\|f_i\right\|_{[\mathcal{H}]^{s}} = \epsilon \sum_{k=1}^m \lambda_{m+k}^{-s}\left(\omega_k^{(i)}\right)^2 \le \epsilon \sum_{k=1}^m \lambda_{2m}^{-s} \leq 2^{s\beta} c \epsilon \sum_{k=1}^m m^{s \beta} \le 2^{s\beta} c \epsilon m^{s \beta+1} = 2^{s\beta} c C_{0},
\end{equation}
Where $c$ in \eqref{proof lower-1} represents the constant in Assumption \ref{ass EDR}. So if $C_{0}$ is small such that 
\begin{equation}\label{C0-1}
    2^{s\beta} c C_{0} \le R, 
\end{equation}
then we have  $f_{i} \in B^{s}(R), i=1,2,\cdots,M.$

Using Lemma \ref{lemma of KL}, we have 
\begin{align}
    \mathrm{KL}\left(\rho_{f_{i}}^n, \rho_{f_{0}}^n\right) &=\frac{n}{2 \bar{\sigma}^{2}}\left\|f_i\right\|_{L^2(\mathcal{X}, \mu)}^2 \notag \\
    &=\frac{n \epsilon}{2 \bar{\sigma}^{2}} \sum_{k=1}^m \left(\omega_k^{(i)}\right)^2 \notag \\
    &\leq \frac{n \epsilon m }{2 \bar{\sigma}^{2}}  = \frac{ n }{2 \bar{\sigma}^{2}} C_{0} m^{-s\beta}.
\end{align}
Recall that $ M \ge 2^{m/8}$ implies $ \ln M \geq \frac{\ln 2}{8} m$. For a fixed $a\in(0,\frac{1}{8})$, since $m = n^{\frac{1}{s \beta +1}}$, letting
\begin{equation}\label{proof 2.8-1}
  \mathrm{KL}\left(\rho_{f_{i}}^n, \rho_{f_{0}}^n\right) \le \frac{ n }{2 \bar{\sigma}^{2}} C_{0} m^{-s\beta} \leq a \frac{\ln 2}{8} m \le a \ln M,
\end{equation}
we have 
\begin{equation}\label{C0-2}
    C_{0} \le \frac{\bar{\sigma}^{2} \ln 2 }{4} a.
\end{equation}
So we can choose $C_{0} = c^{\prime} a$ such that \eqref{C0-1} and \eqref{C0-2} are satisfied.

Denote $ \left\{ \rho_{f_{i}}^n , f_{i} \in B^{s}(R)\right\}$ as a family of probability distribution index by $ f_{i} $, then \eqref{proof 2.8-1} implies the second condition in Lemma \ref{lower prop from tsy} holds. Further, using \eqref{proof lower-2}, we have 
\begin{equation}\label{proof 2.8-2}
    d \left(f_i, f_j\right)^2=\left\|f_i-f_j \right\|_{L^{2}}^2=\epsilon \sum_{k=1}^m\left(\omega_{k}^{(i)} - \omega_{k}^{(j)} \right)^2 \geq \frac{\epsilon m}{8}=\frac{c^{\prime}a}{8} m^{- s \beta} \geq c^{\prime} a n^{-\frac{s \beta}{s \beta + 1}},
\end{equation}
where $c^{\prime}$ is a constant independent of $n$.

Applying Lemma \ref{lower prop from tsy} to \eqref{proof 2.8-1} and \eqref{proof 2.8-2}, we have
\begin{equation}\label{proof lower-3}
\inf _{\hat{f}_n} \sup _{f \in B^s(R)} \mathbb{P}_{\rho_f}\left\{\left\|\hat{f}_n-f\right\|_{L^{2}}^2 \geq c^{\prime} a n^{-\frac{s \beta}{s \beta + 1}}\right\} \geq \frac{\sqrt{M}}{1+\sqrt{M}}\left(1-2 a-\sqrt{\frac{2 a}{\ln M}}\right).
\end{equation}
When $n$ is sufficiently large so that $M$ is sufficiently large, the probability in the R.H.S. of \eqref{proof lower-3} is larger than $ 1- 3a $. For $\delta \in (0,1)$, choose $ a = \frac{\delta}{3}$, without loss of generality we assume $ a \in (0,\frac{1}{8})$. Then \eqref{proof lower-3} shows that there exists a constant $C$, for all estimator $\hat{f}, $ we can find a function $f \in B^{s}(R)$ and the corresponding distribution $\rho_{f} \in \mathcal{P}$ such that, with probability at least $1-\delta$,
\begin{equation}
    \left\| \hat{f} - f \right\|_{L^{2}}^{2} \ge  C \delta n^{-\frac{s  \beta}{s \beta +1}}.
\end{equation}
So we finish the proof.

\section{Useful propositions for upper bounds}

This proposition bounds the $L^{\infty}$ norm of $f_{\lambda} $ when $s \le \alpha_{0}$.
\begin{proposition}\label{prop infty norm}
   Suppose that Assumption \ref{ass EDR}, \ref{assumption embedding} and \ref{ass source condition} hold for $ 0 < s \le \alpha_{0}$ and $\frac{1}{\beta} \le \alpha_{0} < 1$. Denote $f_{\lambda} = (T+\lambda)^{-1}g$, then for any $\lambda >0$ and any $\alpha > \alpha_{0}$, we have 
   \begin{equation}\label{prop appr 2}
       \| f_{\lambda} \|_{L^{\infty}} \le M_{\alpha} \| f_{\rho}^{*} \|_{[\mathcal{H}]^{s}} \lambda^{-\frac{\alpha - s}{2}}.
   \end{equation}
\end{proposition}
\begin{proof}
    Suppose that $ f_{\rho}^{*} = \sum_{i \in N} a_{i} e_{i}$. 
    Since $s \le \alpha_{0}$ and $\alpha > \alpha_{0} $, we have 
    \begin{align}\label{infty-1}
        \left\|f_{\lambda}\right\|_{[\mathcal{H}]^{\alpha}}^{2} &= \sum\limits_{i \in N} \left(\frac{\lambda_{i}^{1-\frac{\alpha}{2}} }{\lambda + \lambda_{i}} \right)^{2}  a_{i}^{2} \notag \\
        &= \sum\limits_{i \in N} \left(\frac{\lambda_{i}^{1-\frac{\alpha-s}{2}} }{\lambda + \lambda_{i}} \right)^{2} \lambda_{i}^{-s} a_{i}^{2} \notag \\
        &\le \| f_{\rho}^{*} \|_{[\mathcal{H}]^s}^{2} \lambda^{s-\alpha},  \notag \\
    \end{align}
    where we use Lemma \ref{lemma freq} for the last inequality. Further, using $\| [\mathcal{H}]^{\alpha} \hookrightarrow L^{\infty}(\mathcal{X}, \mu) \| = M_{\alpha} $ by Assumption \ref{assumption embedding}, we have $ \| f_{\lambda} \|_{L^{\infty}} \le M_{\alpha} \| f_{\lambda} \|_{[\mathcal{H}]^\alpha} \le M_{\alpha} \| f_{\rho}^{*} \|_{[\mathcal{H}]^s}  \lambda^{-\frac{\alpha - s}{2}}$.
    
\end{proof}

The following proposition is an application of the classical Bernstein type inequality but considering a truncation version of $f_{\rho}^{*}$, which will bring refined analysis compared to previous work.
\begin{proposition}\label{theorem 4.9 boundness}
  Suppose that Assumption \ref{ass EDR}, \ref{assumption embedding}, \ref{ass source condition} and \ref{ass mom of error} hold for $ 0 < s \le 2$ and $\frac{1}{\beta} \le \alpha_{0} < 1$. Denote $\xi_{i} = \xi(x_{i},y_{i}) =  T_{\lambda}^{-\frac{1}{2}}(K_{x_{i}} y_{i} - T_{x_{i}} f_{\lambda}) $ and $ \Omega_{0} = \{x \in \Omega: |f_{\rho}^{*}(x)| \le t \}$. Then for any $\alpha > \alpha_{0}$ and all $\delta \in (0,1)$, with probability at least $1 - \delta$, we have
  \begin{equation}
    \left\| \frac{1}{n} \sum\limits_{i=1}^{n} \xi_{i}I_{x_{i} \in \Omega_{0}} - \mathbb{E}\xi_{x}I_{x \in \Omega_{0}} \right\|_{\mathcal{H}} \leq \ln \frac{2}{\delta}\left(\frac{C_1 \lambda^{-\frac{\alpha}{2}}}{n} \cdot \tilde{M} + \frac{C_2 \mathcal{N}^{\frac{1}{2}}(\lambda)}{\sqrt{n}}+\frac{C_3 \lambda^{-\frac{\alpha-s}{2}}}{\sqrt{n}}\right),
    \end{equation}
    where $\tilde{M} = M_{\alpha} R \lambda^{-\frac{\alpha - s}{2}} + t + L$, and $L$ is the constant in (\ref{ass mom of error}). $C_{1} = 8\sqrt{2} M_{\alpha}, C_{2} = 8\sigma, C_{3} = 8\sqrt{2} M_{\alpha} R$.
\end{proposition}
\begin{proof}
   Note that $f_{\rho}^{*}$ can represent a $\mu$-equivalence class in $L^{2}(\mathcal{X},\mu)$. When defining the set $ \Omega_{0}$, we actually denote $f_{\rho}^{*}$ as the representative $f_{\rho}^{*}(x) = \int_{\mathcal{Y}} y \mathrm{d} \rho(y|x).$ 
   
   To use Lemma \ref{bernstein}, we need to bound the m-th moment of $ \xi(x,y) I_{x \in \Omega_{0}} $.
   \begin{align}\label{proof of 4.9-1}
       \mathbb{E} \left\| \xi(x,y) I_{x \in \Omega_{0}} \right\|_{\mathcal{H}}^{m} &= \mathbb{E} \left\| T_{\lambda}^{-\frac{1}{2}} K_{x}(y - f_{\lambda}(x))I_{x \in \Omega_{0}} \right\|_{\mathcal{H}}^{m} \notag \\
       &\le \mathbb{E} \Big( \left\| T_{\lambda}^{-\frac{1}{2}} k(x,\cdot)\right\|_{\mathcal{H}}^{m}  \mathbb{E} \big( \left|(y - f_{\lambda}(x)) I_{x \in \Omega_{0}} \right|^{m} ~\big|~ x\big) \Big).
   \end{align}
Using the inequality $(a+b)^m \leq 2^{m-1}\left(a^m+b^m\right)$, we have 
\begin{align}\label{proof-plug}
    \left|y - f_\lambda(x)\right|^m & \leq 2^{m-1}\left(\left|f_\lambda(x)-f_{\rho}^{*}(x)\right|^m+\left|f_{\rho}^{*}(x)-y\right|^m\right) \notag \\
    & =2^{m-1}\left(\left|f_\lambda(x)-f_{\rho}^{*}(x)\right|^m+|\epsilon|^m\right).
\end{align}
Plugging \eqref{proof-plug} into \eqref{proof of 4.9-1}, we have
\begin{align}
  \mathbb{E} \left\| \xi(x,y)I_{x \in \Omega_{0}} \right\|_{\mathcal{H}}^{m} ~\le~ &2^{m-1} \mathbb{E} \Big( \left\| T_{\lambda}^{-\frac{1}{2}} k(x,\cdot)\right\|_{\mathcal{H}}^{m} \left|(f_{\lambda}(x) - f_{\rho}^{*}(x)) I_{x \in \Omega_{0}} \right|^{m}\Big) \label{proof of 4.9-2} \\
  &+ 2^{m-1} \mathbb{E} \Big( \left\| T_{\lambda}^{-\frac{1}{2}} k(x,\cdot)\right\|_{\mathcal{H}}^{m}  \mathbb{E} \big( \left| \epsilon~ I_{x \in \Omega_{0}} \right|^{m} ~\big|~ x \big) \Big) \label{proof of 4.9-3}
\end{align}

Now we begin to bound (\ref{proof of 4.9-3}). Note that we have proved in Lemma \ref{due embedding bound} that for $\mu$-almost $x \in \mathcal{X}$,
\begin{equation}
  \left\| T_{\lambda}^{-\frac{1}{2}} k(x,\cdot)\right\|_{\mathcal{H}} \le M_{\alpha} \lambda^{-\frac{\alpha}{2}};  
\end{equation}
In addition, we also have
\begin{align}
    \mathbb{E} \left\|T_\lambda^{-\frac{1}{2}} k(x,\cdot)\right\|_\mathcal{H}^{2} &= \mathbb{E}
    \Big\| \sum\limits_{i \in N} ( \frac{1}{\lambda_{i} + \lambda})^{\frac{1}{2}} \lambda_{i} e_{i}(x) e_{i}(\cdot)  \Big\|_{\mathcal{H}}^{2} \notag \\
    &= \mathbb{E} \Big( \sum\limits_{i \in N}  \frac{\lambda_{i}}{\lambda_{i} + \lambda} e_{i}^{2}(x) \Big) \notag \\
    &= \sum\limits_{i \in N}  \frac{\lambda_{i}}{\lambda_{i} + \lambda} \notag \\
    &=\mathcal{N}(\lambda).
\end{align}
So we have 
\begin{equation}
    \mathbb{E}\left\|T_\lambda^{-\frac{1}{2}} k(x,\cdot) \right\|_{\mathcal{H}}^m \leq \sup_{x \in \mathcal{X}} \left\|T_\lambda^{-\frac{1}{2}} k(x,\cdot)\right\|_{\mathcal{H}}^{m-2} \cdot \mathbb{E}\left\|T_\lambda^{-\frac{1}{2}} k(x,\cdot)\right\|_{\mathcal{H}}^2 \leq \big( M_{\alpha} \lambda^{-\frac{\alpha}{2}} \big)^{m-2} \mathcal{N}(\lambda).
\end{equation}
Using Assumption \ref{ass mom of error}, we have
\begin{equation}
    \mathbb{E}\left(|\epsilon I_{x \in \Omega_{0}}|^m  \mid x\right) \le \mathbb{E}\left(|\epsilon|^m \mid x\right) \leq \frac{1}{2} m ! \sigma^2 L^{m-2}, \quad \mu \text {-a.e. } x \in \mathcal{X},
\end{equation}
so we get the upper bound of (\ref{proof of 4.9-3}), i.e., 
\begin{equation}
     (\ref{proof of 4.9-3}) \le \frac{1}{2} m !\left(\sqrt{2} \sigma \mathcal{N}^{\frac{1}{2}}(\lambda)\right)^2(2 M_{\alpha} \lambda^{-\frac{\alpha}{2}} L)^{m-2}.
\end{equation}

Now we begin to bound (\ref{proof of 4.9-2}).
\begin{enumerate}[(1)]
    \item When $s \le \alpha_{0}$, using the definition of $\Omega_{0}$ and Proposition \ref{prop infty norm}, we have 
    \begin{equation}
    \| (f_{\lambda} - f_{\rho}^{*}) I_{x \in \Omega_{0}} \|_{L^{\infty}} \le \| f_{\lambda}  \|_{L^{\infty}} + \| f_{\rho}^{*} I_{x \in \Omega_{0}} \|_{L^{\infty}} \le M_{\alpha} R \lambda^{-\frac{\alpha - s}{2}} + t := M.
    \end{equation}
    \item When $ s > \alpha_{0} $, without loss of generality, we assume $\alpha_{0} < \alpha \le s$. using Theorem \ref{theorem of approximation error} for $\gamma = \alpha$, we have 
    \begin{equation}
    \| (f_{\lambda} - f_{\rho}^{*}) I_{x \in \Omega_{0}} \|_{L^{\infty}} \le M_{\alpha} \| f_{\lambda}  - f_{\rho}^{*} \|_{[\mathcal{H}]^{\alpha}} \le M_{\alpha} R \lambda^{-\frac{\alpha - s}{2}} <  M.
    \end{equation}
\end{enumerate}
Therefore, for all $0 < s \le 2$ we have 
\begin{align}
    \| (f_{\lambda} - f_{\rho}^{*}) I_{x \in \Omega_{0}} \|_{L^{\infty}} \le M.
\end{align}
In addition, using Theorem \ref{theorem of approximation error} for $\gamma=0$, we also have 
\begin{equation}
    \mathbb{E} | (f_{\lambda}(x) - f_{\rho}^{*}(x)) I_{x\in \Omega_{0}} |^{2} \le \mathbb{E} | f_{\lambda}(x) - f_{\rho}^{*}(x)|^{2} \le (R \lambda^{\frac{s}{2}})^{2}.
\end{equation}
So we get the upper bound of (\ref{proof of 4.9-2}), i.e.,
\begin{align}
    (\ref{proof of 4.9-2}) &\le 2^{m-1} (M_{\alpha} \lambda^{-\frac{\alpha}{2}})^{m} \cdot  \| (f_{\lambda} - f_{\rho}^{*}) I_{x \in \Omega_{0}} \|_{L^{\infty}}^{m-2} \cdot \mathbb{E} | (f_{\lambda}(x) - f_{\rho}^{*}(x)) I_{x\in \Omega_{0}} |^{2} \notag \\
    &\le 2^{m-1} (M_{\alpha} \lambda^{-\frac{\alpha}{2}})^{m} \cdot M^{m-2} \cdot (R \lambda^{\frac{s}{2}})^{2} \notag \\
    &\le \frac{1}{2} m! \big( 2 M_{\alpha} \lambda^{-\frac{\alpha}{2}} M \big)^{m-2} \big( 2 M_{\alpha} R \lambda^{-\frac{\alpha - s}{2}}\big)^{2}.
\end{align}
Denote 
\begin{align}
    \tilde{L} &= 2 M_{\alpha} (M + L) \lambda^{-\frac{\alpha}{2}} \notag \\
    \tilde{\sigma} &= 2 M_{\alpha} R \lambda^{-\frac{\alpha-s}{2}} + \sqrt{2} \sigma \mathcal{N}^{\frac{1}{2}}(\lambda),
\end{align}
then we have $\mathbb{E} \left\| \xi(x,y) I_{x\in \Omega_{0}} \right\|_{\mathcal{H}}^{m} \le \frac{1}{2} m! \tilde{\sigma}^{2} \tilde{L}^{m-2} $. Using Lemma \ref{bernstein}, we finish the proof.
\end{proof}
\begin{remark}
  In fact, when we later applying Proposition \ref{theorem 4.9 boundness} in the proof of Proposition \ref{theorem 4.9}, the truncation method in this proposition is necessary only in the $ s \le \alpha_{0}$ case. But for the completeness and consistency of our proof, we also include $s > \alpha_{0}$ in this proposition.
\end{remark}

Based on Proposition \ref{theorem 4.9 boundness}, the following Proposition will give an upper bound of $\left\|T_\lambda^{-\frac{1}{2}}\left[\left(g_Z-T_X f_\lambda\right)-\left(g-T f_\lambda\right)\right]\right\|_{\mathcal{H}}$.
\begin{proposition}\label{theorem 4.9}
  Suppose that Assumption \ref{ass EDR}, \ref{assumption embedding}, \ref{ass source condition} and \ref{ass mom of error} hold for $ 0 < s \le 2$ and $\frac{1}{\beta} \le \alpha_{0} < 1$. Suppose that $f_{\rho}^{*} \in L^{q}(\mathcal{X},\mu)$ and $ \|f_{\rho}^{*} \|_{L^{q}(\mathcal{X},\mu)} \le C_{q} < \infty $ for some $q > \frac{2(s \beta + 1)}{2 + (s-\alpha_{\tiny 0}) \beta}$. Then in the case of $s + \frac{1}{\beta} > \alpha_{0}$, by choosing $ \lambda \asymp n^{-\frac{ \beta}{s \beta + 1}}$, for any fixed $\delta \in (0,1)$, when $n$ is sufficiently large, with probability at least $ 1- \delta$ , we have  
  \begin{equation}\label{goal of theorem 4.9}
      \left\|T_\lambda^{-\frac{1}{2}}\left[\left(g_Z-T_X f_\lambda\right)-\left(g-T f_\lambda\right)\right]\right\|_{\mathcal{H}} \le \ln{\frac{2}{\delta}} C \frac{\lambda^{-\frac{1}{2 \beta}}}{\sqrt{n}} = \ln{\frac{2}{\delta}} C n^{-\frac{1}{2} \frac{s \beta}{s \beta +1}},
  \end{equation}
  where $C$ is a constant that only depends on $ \kappa, R, L, \sigma, C_{q} $.
\end{proposition}
\begin{proof}
   Denote $ \xi_{i} = \xi(x_{i},y_{i}) = T_{\lambda}^{- \frac{1}{2}} (K_{x_{i}} y_{i} - T_{x_i}f_{\lambda} )$ and $ \xi_{x} = \xi(x,y) = T_{\lambda}^{- \frac{1}{2}} (K_{x} y - T_{x}f_{\lambda} )$, then (\ref{goal of theorem 4.9}) is equivalent to
   \begin{equation}\label{equivalent goal}
       \left\| \frac{1}{n} \sum\limits_{i=1}^{n} \xi_{i} - \mathbb{E}\xi_{x} \right\|_{\mathcal{H}} \le \ln{\frac{2}{\delta}} C \frac{\lambda^{-\frac{1}{2 \beta}}}{\sqrt{n}} = \ln{\frac{2}{\delta}} C n^{-\frac{1}{2} \frac{s \beta}{s \beta +1}}.
   \end{equation}
   Consider the subset $\Omega_{1} = \{x \in \Omega: |f_{\rho}^{*}(x)| \le t \}$ and $\Omega_{2} = \mathcal{X} \backslash \Omega_{1}$. Since $\| f_{\rho}^{*} \|_{L^q(\mathcal{X},\mu)} \le C_{q}$, we have
   \begin{equation}
       P(x \in \Omega_{2}) = P\Big(|f_{\rho}^{*}(x)| > t \Big) \le \frac{\mathbb{E} |f_{\rho}^{*}(x)|^{q}}{t^{q}} \le \frac{(C_q)^{q}}{t^q}.
   \end{equation}
   Decomposing $\xi_{i}$ as $\xi_{i} I_{x_{i} \in \Omega_{1} } +  \xi_{i} I_{x_{i} \in \Omega_{2} }$, we have
\begin{align}\label{decomposition}
    \left\|\frac{1}{n} \sum_{i=1}^n \xi_i-\mathbb{E} \xi_x\right\|_\mathcal{H} \le \left\|\frac{1}{n} \sum_{i=1}^n \xi_i I_{x_{i} \in \Omega_{1}}-\mathbb{E} \xi_x I_{x \in \Omega_{1}} \right\|_\mathcal{H} + \left\| \frac{1}{n} \sum_{i=1}^n \xi_i I_{x_{i} \in \Omega_{2}} \right\|_{_\mathcal{H}} + \left\| \mathbb{E} \xi_x I_{x \in \Omega_{2}} \right\|_{_\mathcal{H}}.
\end{align}

Given $s +\frac{1}{\beta} > \alpha_{0}$, here we firstly fixed an $\alpha$ such that 
\begin{equation}\label{require of alpha}
    \alpha_{0} < \alpha < s +\frac{1}{\beta}.
\end{equation}

For the first term in (\ref{decomposition}), denoted as \text{\uppercase\expandafter{\romannumeral1}}, using Theorem \ref{theorem 4.9 boundness}, for all $ \delta \in (0,1)$, with probability at least $1-\delta$, we have
\begin{equation}\label{lead terms 2}
\text{\uppercase\expandafter{\romannumeral1}} \leq \ln \frac{2}{\delta}\left(\frac{C_1 \lambda^{-\frac{\alpha}{2}}}{n}\cdot \tilde{M} +\frac{C_2 \mathcal{N}^{\frac{1}{2}}(\lambda)}{\sqrt{n}}+\frac{C_3 \lambda^{-\frac{\alpha - s}{2}}}{\sqrt{n}}\right),
\end{equation}
where $\tilde{M} = M_{\alpha} R \lambda^{-\frac{\alpha - s}{2}} + t + L$. Recalling that $s +\frac{1}{\beta} > \alpha_{0}$, simple calculation shows that by choosing $ \lambda = n^{-\frac{\beta}{s \beta + 1}}$,
\begin{itemize}
    \item the second term in \eqref{lead terms 2}:
    \begin{align}\label{1.2-4}
        \frac{C_2 \mathcal{N}^{\frac{1}{2}}(\lambda)}{\sqrt{n}} \asymp \frac{\lambda^{-\frac{1}{2 \beta}}}{\sqrt{n}} = n^{-\frac{1}{2} \frac{s \beta}{s \beta +1}};  
    \end{align}
    \item the third term in \eqref{lead terms 2}:
    \begin{equation}\label{1.2-5}
        \frac{C_3 \lambda^{-\frac{\alpha - s}{2}}}{\sqrt{n}} \asymp  n^{\frac{1}{2}(\frac{\alpha}{s + 1 / \beta} - 1)} \cdot n^{-\frac{1}{2} \frac{s \beta}{s \beta +1}} \lesssim  n^{-\frac{1}{2} \frac{s \beta}{s \beta +1}};
    \end{equation}
    \item the first term in \eqref{lead terms 2}:
    \begin{align}\label{1.2-6}
        \frac{C_1 \lambda^{-\frac{\alpha}{2}}}{n}\cdot \tilde{M} &\asymp \frac{\lambda^{-\frac{\alpha}{2}}}{n} \lambda^{-\frac{\alpha - s}{2}} + \frac{\lambda^{-\frac{\alpha}{2}}}{n} \cdot t + \frac{\lambda^{-\frac{\alpha}{2}}}{n} \cdot L.
    \end{align}
\end{itemize}
Further calculations show that
\begin{equation}
    \frac{\lambda^{-\frac{\alpha}{2}}}{n} \lambda^{-\frac{\alpha - s}{2}} = n^{\frac{\alpha}{s + 1 / \beta} - 1} \cdot n^{-\frac{1}{2} \frac{s \beta}{s \beta +1}} \lesssim  n^{-\frac{1}{2} \frac{s \beta}{s \beta +1}},
\end{equation}
and 
\begin{equation}
    \frac{\lambda^{-\frac{\alpha}{2}}}{n} = n^{\frac{1}{2} \frac{\alpha\beta - s \beta -2}{s \beta +1}} \cdot n^{-\frac{1}{2}  \frac{s \beta}{s \beta +1}} \lesssim  n^{-\frac{1}{2} \frac{s \beta}{s \beta +1}}.
\end{equation}
To make $\eqref{1.2-6} \lesssim n^{-\frac{1}{2} \frac{s \beta}{s \beta +1}}$ when $ \lambda = n^{-\frac{\beta}{s \beta + 1}}$, letting $ \frac{\lambda^{-\frac{\alpha}{2}}}{n} \cdot t \le n^{-\frac{1}{2} \frac{s \beta}{s \beta +1}}$, we have the first restriction of $t$:
\begin{equation}\label{restrict1}
    \textbf{(R1)}:\quad t \le n^{\frac{1}{2} (1+\frac{1-\alpha \beta}{s \beta + 1})}.
\end{equation}
That is to say, if we choose $ t \le n^{\frac{1}{2} (1+\frac{1-\alpha \beta}{s \beta + 1})} $, we have 
\begin{equation}
    \text{\uppercase\expandafter{\romannumeral1}} \le \ln{\frac{2}{\delta}} C \frac{\lambda^{-\frac{1}{2 \beta}}}{\sqrt{n}} = \ln{\frac{2}{\delta}} C n^{-\frac{1}{2} \frac{s \beta}{s \beta +1}}.
\end{equation}

For the second term in (\ref{decomposition}), denoted as \text{\uppercase\expandafter{\romannumeral2}}, we have 
\begin{align}
    \tau_{n} := P(\text{\uppercase\expandafter{\romannumeral2}} > \frac{ \lambda^{-\frac{1}{2 \beta}}}{\sqrt{n}}) 
    &\le P\Big( ~\exists x_{i} ~\text{s.t.}~ x_{i} \in \Omega_{2} \Big) = 1 - P\Big(x_{i} \notin \Omega_{2}, \forall i=1,2,\cdots,n \Big) \notag \\
    &= 1 - P\Big(x \notin \Omega_{2}\Big)^{n} \notag \\
    &= 1 - P\Big( |f_{\rho}^{*}(x)| \le t\Big)^{n} \notag \\
    & \le 1 - \Big( 1 - \frac{(C_q)^{q}}{t^{q}}\Big)^{n}.
\end{align}
Letting $\tau_{n} := P(\text{\uppercase\expandafter{\romannumeral2}} > \frac{ \lambda^{-\frac{1}{2 \beta}}}{\sqrt{n}}) \to 0$, we have $ \displaystyle{t^{q}} \gg n$, i.e. $t \gg n^{\frac{1}{q}} $. This gives the second restriction of $t$, i.e., 
\begin{equation}\label{restrict2}
   \textbf{(R2)}:\quad t \gg n^{\frac{1}{q}}, ~\text{or}~ n^{\frac{1}{q}} = o(t).
\end{equation}

For the third term in (\ref{decomposition}), denoted as \uppercase\expandafter{\romannumeral3}. Since we have already known that $\| T_{\lambda}^{-\frac{1}{2}} k(x,\cdot)\|_{\mathcal{H}} \le \lambda^{-\frac{\alpha}{2}}, \mu \text {-a.e. } x \in \mathcal{X},$  so
\begin{align}\label{third term}
    \text{\uppercase\expandafter{\romannumeral3}} &\le \mathbb{E}\| \xi_{x} I_{x\in\Omega_{2}} \|_{\mathcal{H}} \le \mathbb{E}\Big[ \| T_{\lambda}^{-\frac{1}{2}} k(x,\cdot) \|_{\mathcal{H}} \cdot \big| \big(y-f_{\lambda}(x) \big) I_{x\in\Omega_{2}}\big| \Big] \notag \\
    &\le \lambda^{-\frac{\alpha}{2}} \mathbb{E} \big| \big(y-f_{\lambda}(x) \big) I_{x\in\Omega_{2}}\big| \notag \\
    &\le \lambda^{-\frac{\alpha}{2}} \Big( \mathbb{E} \big| \big(f_{\rho}^{*}(x)-f_{\lambda}(x) \big) I_{x\in\Omega_{2}}\big| +  \mathbb{E} \big| \big(f_{\rho}^{*}(x)-y \big) I_{x\in\Omega_{2}}\big| \Big) \notag \\
    &\le \lambda^{-\frac{\alpha}{2}} \Big( \mathbb{E} \big| \big(f_{\rho}^{*}(x)-f_{\lambda}(x) \big) I_{x\in\Omega_{2}}\big| +  \mathbb{E} \big| \epsilon \cdot I_{x\in\Omega_{2}}\big| \Big). 
\end{align}
Using Cauchy-Schwarz and the bound of approximation error (Theorem \ref{theorem of approximation error}), we have
\begin{align}\label{1.2-1}
    \mathbb{E} \big| \big(f_{\rho}^{*}(x)-f_{\lambda}(x) \big) I_{x\in\Omega_{2}}\big| \le \left( \left\|f_{\rho}^*-f_\lambda\right\|_{L^{2}}\right)^{\frac{1}{2}} \cdot \left(P(x \in \Omega_{2})\right)^{\frac{1}{2}} \le R \lambda^{\frac{s}{2}} C_{q}^{\frac{q}{2}} t^{-\frac{q}{2}}.
\end{align}
In addition, we have
\begin{align}\label{1.2-2}
    \mathbb{E} \big| \epsilon \cdot I_{x\in\Omega_{2}}\big| = \mathbb{E} \left( \mathbb{E} \big| \epsilon \cdot I_{x\in\Omega_{2}}\big| ~\Big|~ x\right) \le \sigma \mathbb{E} \left| I_{x\in\Omega_{2}}\right| \le \sigma (C_{q})^{q} t^{-q}.
\end{align}
Plugging \eqref{1.2-1} and \eqref{1.2-2} into \eqref{third term}, we have
\begin{equation}\label{1.2-3}
    \text{\uppercase\expandafter{\romannumeral3}} \le R C_{q}^{\frac{q}{2}} \lambda^{-\frac{\alpha - s }{2}} t^{-\frac{q}{2}} + \sigma (C_{q})^{q}  \lambda^{-\frac{\alpha}{2}} t^{-q}.
\end{equation}

Comparing (\ref{1.2-3}) with $C_3 \frac{ \lambda^{-\frac{\alpha - s}{2}}}{\sqrt{n}}$ and $C_1 \frac{\lambda^{-\frac{\alpha}{2}}}{n}$ in (\ref{lead terms 2}). We know that if $ t \ge n^{\frac{1}{q}}$, (\ref{third term}) $\le C \frac{\lambda^{-\frac{1}{2 \beta}}}{\sqrt{n}} = C n^{-\frac{1}{2} \frac{s \beta}{s \beta +1}}$. So the third term will not give further restriction of $t$.

To sum up, if we choose $t$ such that restrictions (\ref{restrict1}) and (\ref{restrict2}) are satisfied, then we can prove that (\ref{equivalent goal}) is satisfied with probability at least $ 1 - \delta - \tau_{n}, (\tau_{n} \to 0)$. Since for a fixed $\delta \in (0,1)$, when $n$ is sufficiently large, $\tau_{n}$ is sufficiently small such that, e.g., $\tau_{n} < \frac{\delta}{10}$. Without loss of generality, we say \eqref{equivalent goal} is satisfied with probability at least $ 1 - \delta$.  

Note that, such $t$ exists if 
\begin{equation}
    \frac{1}{q} < \frac{1}{2} (1+\frac{1-\alpha \beta}{s \beta + 1}) \Longleftrightarrow q > \frac{2(s \beta + 1)}{2 + (s-\alpha) \beta}.
\end{equation}
Recalling that for \eqref{require of alpha}, we only assume there exists $ \alpha $ satisfying $\alpha_{0} < \alpha < s +\frac{1}{\beta}$, so such $t$ exists if and only if 
\begin{equation}
    \frac{1}{q} < \frac{1}{2} (1+\frac{1-\alpha_{0} \beta}{s \beta + 1}) \Longleftrightarrow q > \frac{2(s \beta + 1)}{2 + (s-\alpha_{0}) \beta}.
\end{equation}
which is what we assume in the theorem.
\end{proof}

\begin{proposition}\label{lemma4.6}
Suppose that the embedding index is $\alpha_{0}$. Then for any $\alpha > \alpha_{0}$ and all $\delta \in (0,1)$, with probability at least $1 - \delta$, we have
   \begin{equation}
        \Vert T_\lambda^{-\frac{1}{2}} (T - T_X) T_\lambda^{-\frac{1}{2}} \Vert
        \le \frac{4 M_{\alpha}^{2} \lambda^{-\alpha}}{3n} B + \sqrt {\frac{2 M_{\alpha}^{2} \lambda^{-\alpha}}{n} B},
   \end{equation}
   where 
   \begin{equation}
       B = \ln{\frac{4 \mathcal{N}(\lambda) (\|T\| + \lambda) }{\delta \|T\|}}.
   \end{equation}
\end{proposition}
   
\begin{proof}
  Denote $A_{i} = T_\lambda^{-\frac{1}{2}} (T - T_{x_{i}}) T_\lambda^{-\frac{1}{2}} $, using Lemma \ref{emb norm} we have 
  \begin{equation}
      \| A_{i} \| = \| T_\lambda^{-\frac{1}{2}} T T_\lambda^{-\frac{1}{2}} \| + \| T_\lambda^{-\frac{1}{2}} T_{x_{i}} T_\lambda^{-\frac{1}{2}} \| \le 2 M_{\alpha}^{2} \lambda^{-\alpha}.
  \end{equation}
  We use $ A \preceq B$ to denote that $ A-B$ is a positive semi-definite operator. Using the fact that $\mathbb{E}(B-\mathbb{E} B)^2 \preceq \mathbb{E} B^2$ for a self-adjoint operator $B$, we have
  \begin{equation}
    \mathbb{E} A_{i}^{2} \preceq \mathbb{E}\left[T_\lambda^{-\frac{1}{2}} T_{x_{i}} T_\lambda^{-\frac{1}{2}}\right]^2.
    \end{equation}
  In addition, Lemma \ref{emb norm} shows that $0 \preceq T_\lambda^{-\frac{1}{2}} T_{x_{i}} T_\lambda^{-\frac{1}{2}} \preceq M_{\alpha}^{2}\lambda^{-\alpha}, \mu \text {-a.e. } x \in \mathcal{X}$. So we have
  \begin{equation}
     \mathbb{E} A_{i}^{2}  \preceq \mathbb{E} \left[T_\lambda^{-\frac{1}{2}} T_{x_{i}} T_\lambda^{-\frac{1}{2}}\right]^{2} \preceq \mathbb{E}\left[ M_{\alpha}^{2} \lambda^{-\alpha} \cdot T_\lambda^{-\frac{1}{2}} T_{x_{i}} T_\lambda^{-\frac{1}{2}}\right] = M_{\alpha}^{2}\lambda^{-\alpha} T_{\lambda}^{-1} T,
  \end{equation}
  Defining an operator $V := M_{\alpha}^{2} \lambda^{-\alpha} T_{\lambda}^{-1} T$, we have  
  \begin{align}
      \| V \| &= M_{\alpha}^{2} \lambda^{-\alpha} \frac{\lambda_{1}}{\lambda_{1} + \lambda} = M_{\alpha}^{2} \lambda^{-\alpha} \frac{\|T\|}{\|T\| + \lambda} \le M_{\alpha}^{2} \lambda^{-\alpha}; \notag \\
      \text{tr}V &= M_{\alpha}^{2} \lambda^{-\alpha} \mathcal{N}(\lambda); \notag \\
      \frac{\text{tr}V}{ \| V \|} &= \frac{\mathcal{N}(\lambda) (\|T\| + \lambda)}{\|T\|}. 
  \end{align}
  Using Lemma \ref{lemma concentration of operator} to $A_{i}$, $V$, we finish the proof.
\end{proof}

\section{Auxiliary lemma}
\subsection{Lemmas for upper bound}
The following lemma is where we take advantage of the embedding index and embedding property in Assumption \ref{assumption embedding}.
\begin{lemma}\label{due embedding bound}
   Suppose that the embedding index is $\alpha_{0}$. Then for any $\alpha > \alpha_{0}$, for $\mu$-almost $x \in \mathcal{X}$, we have 
   \begin{align}\label{bound of Tk}
      \|T_{\lambda}^{-\frac{1}{2}} k(x,\cdot) \|_{\mathcal{H}}^{2} \le M_{\alpha}^{2} \lambda^{-\alpha}.
  \end{align}
\end{lemma}
\begin{proof}
   Recalling that $ \| [\mathcal{H}]^{\alpha} \hookrightarrow L^{\infty}(\mathcal{X}) \| = M_{\alpha} $, we have
   \begin{align}
      \|T_{\lambda}^{-\frac{1}{2}} k(x,\cdot) \|_{\mathcal{H}}^{2} &= \Big\| \sum\limits_{i \in N} ( \frac{1}{\lambda_{i} + \lambda})^{\frac{1}{2}} \lambda_{i} e_{i}(x) e_{i}(\cdot)  \Big\|_{\mathcal{H}}^{2} \notag \\
      &=  \sum\limits_{i \in N}  \frac{\lambda_{i}}{\lambda_{i} + \lambda} e_{i}^{2}(x) \notag \\
      &= \big[ \sum\limits_{i \in N}  \lambda_{i}^{\alpha} e_{i}^{2}(x) \big] \sup\limits_{i \in N} \frac{\lambda_{i}^{1-\alpha}}{\lambda_{i} + \lambda} \notag \\
      & \le M_{\alpha}^{2} \lambda^{-\alpha},
  \end{align}
  where we use Lemma \ref{lemma freq} for the last inequality, and we finish the proof.
\end{proof}

Lemma \ref{due embedding bound} has a direct corollary.
\begin{corollary}\label{emb norm}
   Suppose that the embedding index is $\alpha_{0}$. Then for any $\alpha > \alpha_{0}$, for $\mu$-almost $x \in \mathcal{X}$, we have
\begin{equation}
    \| T_{\lambda}^{-\frac{1}{2}} T_{x} T_{\lambda}^{-\frac{1}{2}}\| \le M_{\alpha}^{2} \lambda^{-\alpha}.
\end{equation}  
\end{corollary}

\begin{proof}
  Note that for any $f \in \mathcal{H}$,
  \begin{align}
      T_{\lambda}^{-\frac{1}{2}} T_{x} T_{\lambda}^{-\frac{1}{2}} f &= T_{\lambda}^{-\frac{1}{2}} K_{x} K_{x}^{*}  T_{\lambda}^{-\frac{1}{2}} f \notag \\
      &= T_{\lambda}^{-\frac{1}{2}} K_{x} \langle k(x,\cdot), T_{\lambda}^{-\frac{1}{2}} f \rangle_{\mathcal{H}} \notag \\
      &= T_{\lambda}^{-\frac{1}{2}} K_{x} \langle T_{\lambda}^{-\frac{1}{2}} k(x,\cdot),  f \rangle_{\mathcal{H}} \notag \\
      &=  \langle T_{\lambda}^{-\frac{1}{2}} k(x,\cdot),  f \rangle_{\mathcal{H}} \cdot T_{\lambda}^{-\frac{1}{2}} k(x,\cdot).
  \end{align}
  So $\| T_{\lambda}^{-\frac{1}{2}} T_{x} T_{\lambda}^{-\frac{1}{2}} \| = \sup\limits_{\| f\|_{\mathcal{H}}=1} \| T_{\lambda}^{-\frac{1}{2}} T_{x} T_{\lambda}^{-\frac{1}{2}} f\|_{\mathcal{H}} = \sup\limits_{\| f\|_{\mathcal{H}}=1} \langle T_{\lambda}^{-\frac{1}{2}} k(x,\cdot),  f \rangle_{\mathcal{H}} \cdot \|T_{\lambda}^{-\frac{1}{2}} k(x,\cdot) \|_{\mathcal{H}} = \|T_{\lambda}^{-\frac{1}{2}} k(x,\cdot) \|_{\mathcal{H}}^{2}$. 
  Use Lemma \ref{due embedding bound} and we finish the proof.
\end{proof}
The following concentration inequality about self-adjoint Hilbert-Schmidt operator
valued random variables is frequently used in related literature, e.g., \citet[Theorem 27]{fischer2020_SobolevNorm} and \citet[Lemma 26]{lin2020_OptimalConvergence}.
\begin{lemma}\label{lemma concentration of operator}
   Let $(\mathcal{X}, \mathcal{B}, \mu)$ be a probability space, $\mathcal{H}$ be a separable Hilbert space. Suppose that $ A_{1}, \cdots, A_{n}$ are i.i.d. random variables with values in the set of self-adjoint Hilbert-Schmidt operators. If  $\mathbb{E} A_{i} = 0$, and the operator norm $ \| A_{i} \| \le L ~~ \mu \text {-a.e. } x \in \mathcal{X}$, and there exists a self-adjoint positive semi-definite trace class operator $V$ with $\mathbb{E} A_{i}^{2} \preceq V $. Then for $\delta \in (0,1)$, with probability at least $1 - \delta$, we have 
   \begin{align}
        \left\| \frac{1}{n}\sum_{i=1}^n A_i \right\|
        \leq \frac{2L\beta}{3n} + \sqrt {\frac{2 \| V \| \beta}{n}},\quad
        \beta = \ln \frac{4 \rm{tr} V}{\delta \| V \|}.
   \end{align}
\end{lemma}

The following Bernstein inequality about vector-valued random variables is frequently used, e.g., \citet[Proposition 2]{Caponnetto2007OptimalRF} and \citet[Theorem 26]{fischer2020_SobolevNorm}.
\begin{lemma}[Bernstein inequality]\label{bernstein}
   Let $(\Omega,\mathcal{B},P)$ be a probability space, $H$ be a separable Hilbert space, and $\xi: \Omega \to H$ be a random variable with 
   \begin{displaymath}
     \mathbb{E}\|\xi\|_H^m \leq \frac{1}{2} m ! \sigma^2 L^{m-2},
   \end{displaymath}
   for all $m>2$. Then for $\delta \in (0,1)$, $\xi_{i}$ are i.i.d. random variables, with probability at least $1 - \delta$, we have
   \begin{equation}
       \left\|\frac{1}{n} \sum_{i=1}^n \xi_{i} - \mathbb{E} \xi\right\|_H \le 4\sqrt{2} \ln{\frac{2}{\delta}} \left(\frac{L}{n} + \frac{\sigma}{\sqrt{n}}\right).
   \end{equation}
\end{lemma}
\begin{lemma}\label{lemma of effect}
    If $\lambda_i \asymp i^{-\beta}$, we have
    \begin{align}
        \mathcal{N}(\lambda) \asymp \lambda^{-\frac{1}{\beta}}.
    \end{align}

\end{lemma}
\begin{proof}
    Since $c ~i^{-\beta} \leq \lambda_i \leq C i^{-\beta}$, we have
    \begin{align}
        \mathcal{N}(\lambda) &= \sum_{i = 1}^{\infty}  \frac{\lambda_i}{\lambda_i + \lambda}  
        \leq \sum_{i = 1}^{\infty} \frac{C i^{-\beta}}{C i^{-\beta} + \lambda}  = \sum_{i = 1}^{\infty}  \frac{C }{C+ \lambda i^{\beta}}  \\
        &\leq \int_{0}^{\infty}  \frac{C }{\lambda x^{\beta} + C}  \mathrm{d} x
        = \lambda^{-\frac{1}{\beta}} \int_{0}^{\infty}  \frac{C }{y^{\beta} + C} \mathrm{d} y \leq C_{1} \lambda^{-\frac{1}{\beta}}.
    \end{align}
    for some constant $C_{1}$. Similarly, we have 
    \begin{equation}
    \mathcal{N}(\lambda) \geq C_{2} \lambda^{-\frac{1}{\beta}},
    \end{equation}
    for some constant $C_{2}$.
\end{proof}
\subsection{Lemmas for minimax lower bound}
The following lemma is a standard approach to derive the minimax lower bound, which can be found in \citet[Theorem 2.5]{tsybakov2009_IntroductionNonparametric}. 
\begin{lemma}\label{lower prop from tsy}
Suppose that there is a non-parametric class of functions $ \Theta$ and a (semi-)distance $d(\cdot,\cdot)$ on $ \Theta$. $\left\{ P_{\theta}, \theta \in \Theta \right\}$ is a family of probability distributions indexed by $\Theta$. Assume that $M \ge 2$ and suppose that $ \Theta$ contains elements $ \theta_0, \theta_1, \cdots, \theta_M$ such that, 
\begin{enumerate}[(1)]
    \item $ d\left(\theta_j, \theta_k\right) \geq 2 s>0, \quad \forall 0 \leq j<k \leq M$;
    \item $P_j \ll P_0, \quad \forall j=1, \cdots, M$, and 
    \begin{equation}
        \frac{1}{M } \sum_{j=1}^M K\left(P_j, P_0\right) \leq a \log M,
    \end{equation}
\end{enumerate}
    with $ 0<a<1 / 8$ and $ P_j=P_{\theta_j}, j=0,1, \cdots, M$. Then
    \begin{equation}
    \inf _{\hat{\theta}} \sup _{\theta \in \Theta} P_\theta(d(\hat{\theta}, \theta) \geq s) \geq \frac{\sqrt{M}}{1+\sqrt{M}}\left(1-2 a-\sqrt{\frac{2 a}{\log M}}\right).
    \end{equation}
\end{lemma}

\begin{lemma}\label{lemma of KL}
   Suppose that $\mu$ is a distribution on $\mathcal{X}$ and $f_{i} \in L^{2}(\mathcal{X},\mu)$. Suppose that
   \begin{equation}
       y=f_i(x)+\epsilon, \quad i=1,2,
   \end{equation}
   where $\epsilon \sim \mathcal{N}(0,\sigma^{2})$ are independent Gaussian random error. Denote the two corresponding distributions on $ \mathcal{X} \times \mathcal{Y}$ as $ \rho_{i}, i=1,2$. The KL divergence of two probability distributions on $\Omega$ is 
   \begin{equation}
       K\left(P_1, P_2\right) \coloneqq \int_{\Omega} \log \left(\frac{\mathrm{d} P_1}{\mathrm{~d} P_2}\right) \mathrm{d} P_1,
   \end{equation}
   if $P_1 \ll P_2$ and otherwise $K\left(P_1, P_2\right) \coloneqq \infty $.
   Then we have 
   \begin{equation}
       \mathrm{KL}\left(\rho_1^n, \rho_2^n\right)=n \mathrm{KL}\left(\rho_1, \rho_2\right)=\frac{n}{2 \sigma^2}\left\|f_1-f_2\right\|_{L^2(\mathcal{X}, d \mu)}^2,
   \end{equation}
   where $ \rho_{i}^{n} $ denotes the independent product of $n$ distributions $\rho_{i}, i=1,2$.
\end{lemma}
\begin{proof}
The lemma directly follows from the definition of KL divergence and the fact that 
\begin{equation}
    \mathrm{KL}\left(N\left(f_1(x), \sigma^2\right), N\left(f_2(x), \sigma^2\right)\right) = \frac{1}{2 \sigma^2}\left|f_1(x)-f_2(x)\right|^2.
\end{equation}

\end{proof}

The following lemma is a result from \citet[Lemma 2.9]{tsybakov2009_IntroductionNonparametric}
\begin{lemma}\label{lemma of ham}
   Denote $\Omega=\left\{\omega=\left(\omega_1, \cdots, \omega_m\right), \omega_i \in\{0,1\}\right\}=\{0,1\}^m$. Let $m\ge 8$, there exists a subset $\left\{\omega^{(0)}, \cdots, \omega^{(M)}\right\} $ of ~$ \Omega$ such that $\omega^{(0)}=(0, \cdots, 0)$,
   \begin{equation}
       d_{\text {Ham }}\left(\omega^{(i)}, \omega^{(j)}\right) \coloneqq \sum_{k=1}^m\left|\omega_k^{(i)}-\omega_k^{(j)}\right| \geq \frac{m}{8}, \quad \forall 0 \leq i<j \leq M,
   \end{equation}
   and $M \geq 2^{m / 8}$.
\end{lemma}

\section{Details of experiments}\label{appendix detail experiments}

\subsection{Experiments in Sobolev RKHS}
First, we prove that the series in \eqref{series of sobolev} converges and $f^{*}(x)$ is continuous on $(0,1)$ for $0 < s < \frac{1}{\beta} = 0.5$.

We begin with the computation of the sum of first $N$ terms of $\{ \sin 2 k \pi x + \cos 2 k \pi x \}$, note that 
\begin{align}
    &-2 \sin(\pi x) \left( \sin\left(2 \pi x \right) + \sin\left(4 \pi x\right) + \cdots + \sin\left(2 N \pi x\right) \right) \notag \\
    &= \left[ \cos\left(2 \pi + \pi \right)x - \cos\left(2 \pi - \pi \right)x \right] + \left[ \cos\left( 4 \pi + \pi \right)x - \cos\left(4 \pi - \pi\right)x \right] \notag \\
    &\quad \quad + \cdots + \left[ \cos\left( 2 N \pi + \pi \right)x - \cos\left(2 N \pi - \pi\right)x \right] \notag \\
    &= \cos\left( 2 N \pi + \pi \right)x - \cos \pi x.
\end{align}
So we have 
\begin{align}\label{sin-1}
    \left| \left( \sin\left(2 \pi x \right) + \sin\left(4 \pi x\right) + \cdots + \sin\left(2 N \pi x\right) \right) \right| = \frac{\left| \cos\left( 2 N \pi + \pi \right)x - \cos \pi x \right|}{\left| 2 \sin(\pi x) \right|};
\end{align}
Similarly, we have
\begin{align}\label{cos-2}
    \left| \left( \cos\left(2 \pi x \right) + \cos\left(4 \pi x\right) + \cdots + \cos\left(2 N \pi x\right) \right) \right| = \frac{\left| \sin\left( 2 N \pi + \pi \right)x - \sin \pi x \right|}{\left| 2 \sin(\pi x) \right|}.
\end{align}
Note that \eqref{sin-1} and \eqref{cos-2} are uniformly bounded in $[\delta_{0}, 1-\delta_{0}]$ for any $\delta_{0} > 0$ and $N$. In addition, $\{ k ^{-(s + 0.5)} \}$ is monotone and decreases to zero. Use the Dirichlet criterion and we know that the series in \eqref{series of sobolev} is uniformly convergence in $[\delta_{0}, 1-\delta_{0}]$. Due to the arbitrariness of $\delta_{0}$, we know that the series converges and $f^{*}(x)$ is continuous on $(0,1)$.

In Figure \ref{figure appendix sob_allc} (a), we present the results of different choices of $c$ for $\lambda = c n^{-\frac{\beta}{s \beta + 1}}$ in the experiment of Section \ref{section RKHS experiments}. Figure \ref{figure 1} corresponds to the curve $c=0.1$ in Figure \ref{figure appendix sob_allc} (a), which has the smallest generalization error. In Figure \ref{figure appendix sob_allc} (b), we use 5-fold cross validation to choose the regularization parameter in KRR and present the logarithmic errors and sample sizes. Again, we use logarithmic least-squares to compute the convergence rate $r$, which is still approximately equal to $ n^{-\frac{s \beta}{s \beta + 1}} = n^{-\frac{4}{9}}$.

\begin{figure}[htbp]
\vskip 0.05in
\centering
\subfigure[fixed c]{\includegraphics[width=0.45\columnwidth]{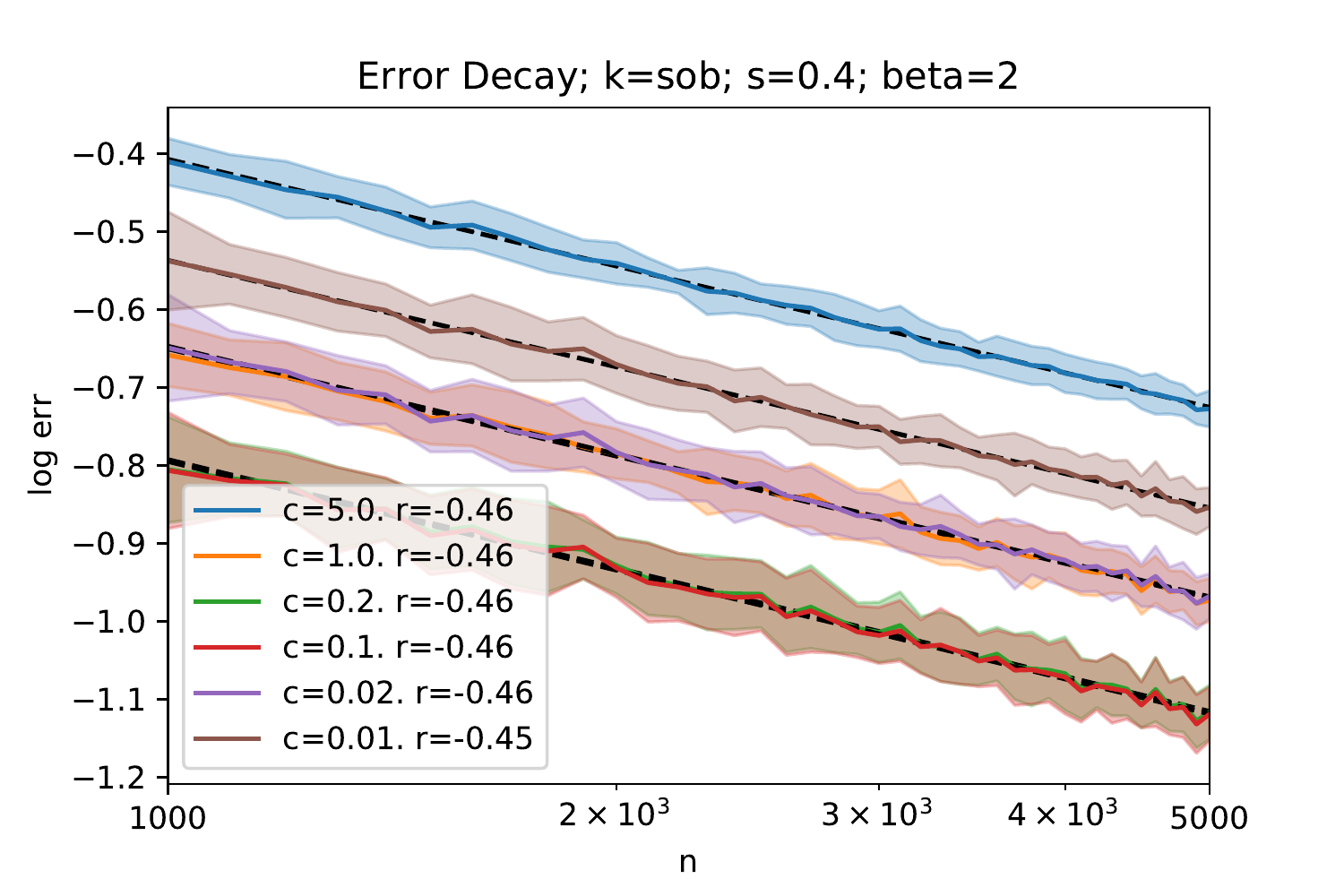}}
\subfigure[cross validation]{\includegraphics[width=0.45\columnwidth]{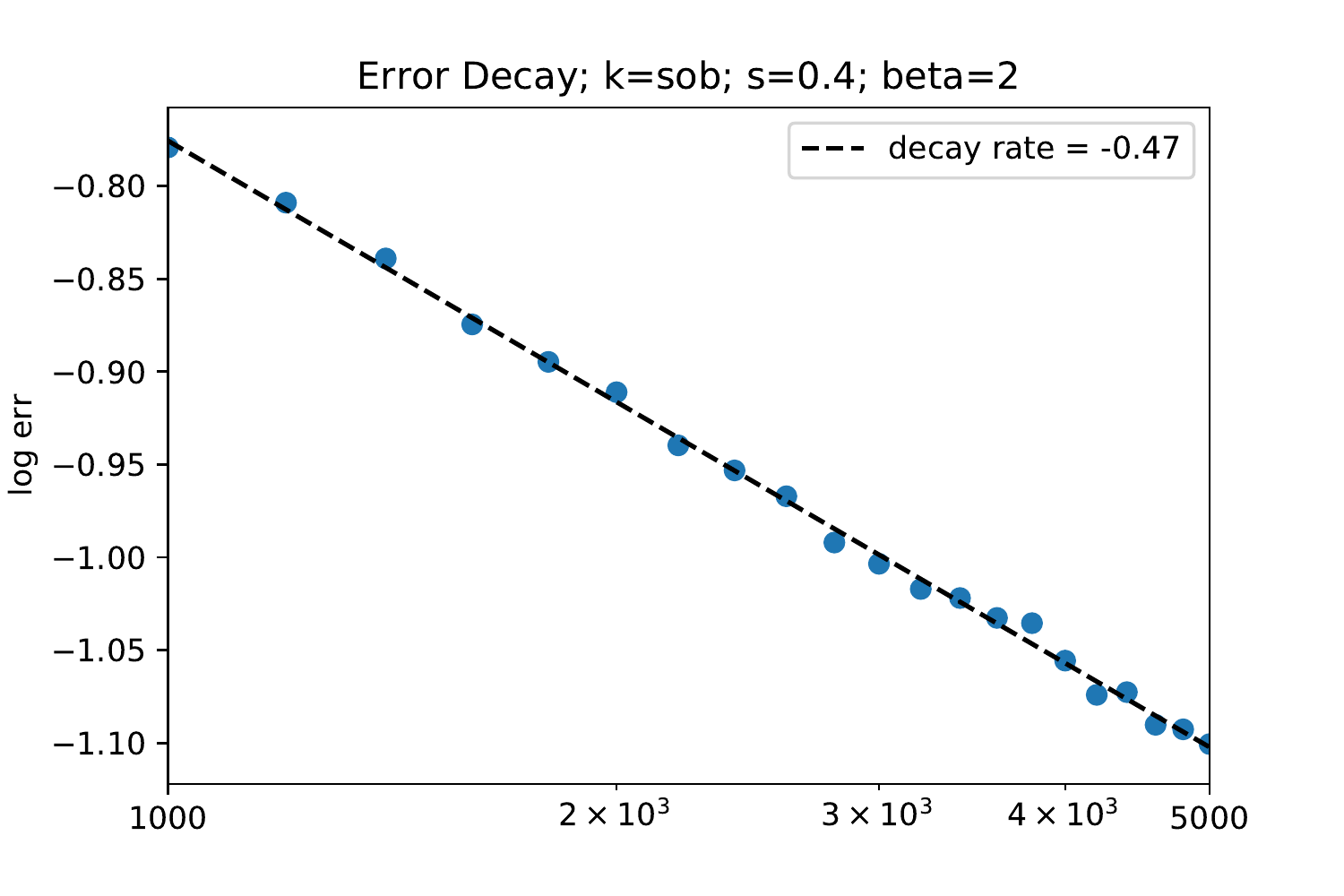}}
\caption{ Error decay curves of Sobolev RKHS. Both axes are logarithmic. (a) The curves show the average generalization errors of different c over 50 trials; and the regions within one standard deviation are shown in the corresponding colors. (b) The scatters show the average generalization errors obtained by 5-fold cross validation over 50 trials. In both (a) and (b), the dashed black lines are computed using logarithmic least-squares, and the slopes represent the convergence rates $r$.}
\label{figure appendix sob_allc}
\vskip 0.05in
\end{figure}


\subsection{Experiments in general RKHS}
First, we prove that the series in \eqref{series of min} converges and $f^{*}(x)$ is continuous on $(0,1)$ for $0 < s < \frac{1}{\beta} = 0.5$.

We begin with the computation of the sum of first $N$ terms of $e_{2k-1}(x) $,
\begin{align}
    &-2 \sin(\pi x) \left( \sin\left(\frac{\pi x}{2}\right) + \sin\left(\frac{5 \pi x}{2}\right) + \cdots + \sin\left(\frac{\left(4N - 3\right)\pi x}{2}\right) \right) \notag \\
    &= \left[ \cos\left( \pi + \frac{\pi}{2}\right)x - \cos\left( \pi - \frac{\pi}{2}\right)x \right] + \left[ \cos\left( \frac{5 \pi}{2} + \pi \right)x - \cos\left(\frac{5 \pi}{2} - \pi\right)x \right] \notag \\
    &\quad \quad + \cdots + \left[ \cos\left( \frac{\left(4N - 3\right) \pi}{2} + \pi \right)x - \cos\left(\frac{\left(4N - 3\right) \pi}{2} - \pi\right)x \right] \notag \\
    &= \cos\left( \frac{\left(4N - 1\right) \pi}{2}\right)x - \cos\frac{\pi}{2}x.
\end{align}
So we have 
\begin{align}
    \left| \sin\left(\frac{\pi x}{2}\right) + \sin\left(\frac{5 \pi x}{2}\right) + \cdots + \sin\left(\frac{\left(4N - 3\right)\pi x}{2}\right) \right| = \frac{\left| \cos\left( \frac{\left(4N - 1\right) \pi}{2}\right)x - \cos\frac{\pi}{2}x \right|}{\left| 2 \sin(\pi x) \right|},
\end{align}
which is uniformly bounded in $[\delta_{0}, 1-\delta_{0}]$ for any $\delta_{0} > 0$ and $N$.

Note that $\{ k ^{-(s + 0.5)} \}$ is monotone and decreases to zero. Use the Dirichlet criterion and we know that the series in \eqref{series of min} is uniformly convergence in $[\delta_{0}, 1-\delta_{0}]$. Due to the arbitrariness of $\delta_{0}$, we know that the series converges and $f^{*}(x)$ is continuous on $(0,1)$.

In Figure \ref{figure appendix min_allc} (a), we present the results of different choices of $c$ for $\lambda = c n^{-\frac{\beta}{s \beta + 1}}$ in the experiment of Section \ref{section RKHS experiments}. Figure \ref{figure 1} corresponds to the curve $c=1.0$ in Figure \ref{figure appendix min_allc} (a), which has the smallest generalization error. In Figure \ref{figure appendix min_allc} (b), we use 5-fold cross validation to choose the regularization parameter in KRR and present the logarithmic errors and sample sizes. Again, we use logarithmic least-squares to compute the convergence rate $r$, which is still approximately equal to $ n^{-\frac{s \beta}{s \beta + 1}} = n^{-\frac{4}{9}}$.
\begin{figure}[htbp]
\vskip 0.1in
\centering
\subfigure[fixed c]{\includegraphics[width=0.45\columnwidth]{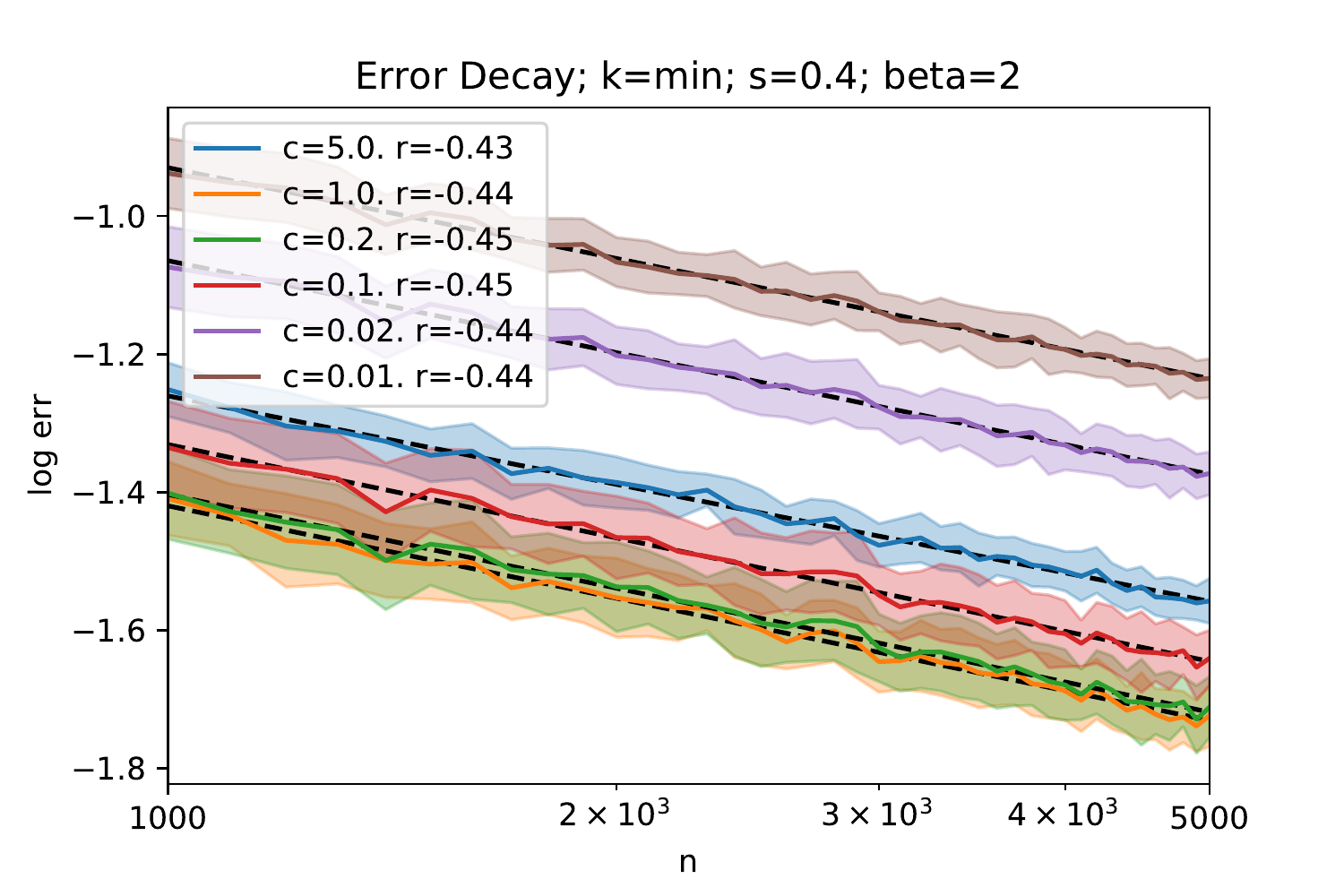}}
\subfigure[cross validation]{\includegraphics[width=0.45\columnwidth]{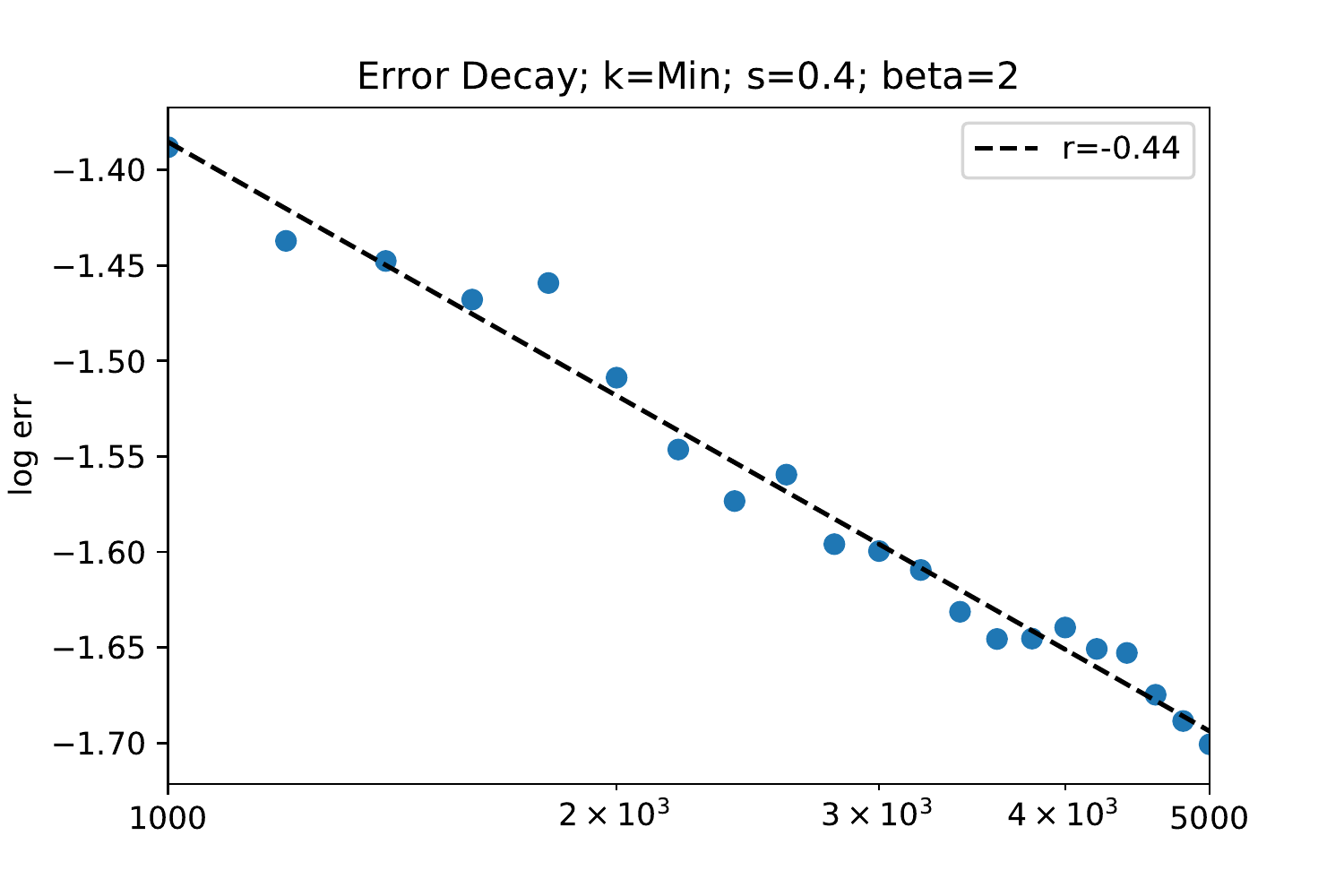}}
\caption{ Error decay curves of general RKHS.}
\label{figure appendix min_allc}
\vskip 0.1in
\end{figure}

\end{document}